\documentclass[twoside]{article}

\usepackage[accepted]{aistats2023}

\usepackage[utf8]{inputenc} 
\usepackage[T1]{fontenc}    
\usepackage{url}            
\usepackage{booktabs}       
\usepackage{enumitem}
\usepackage{enumerate}
\usepackage{amsfonts}       
\usepackage{nicefrac}       
\usepackage{microtype}      
\usepackage{xcolor}         
\usepackage{enumitem}
\usepackage{todonotes}
\usepackage{amssymb,fge}
\usepackage{thm-restate}
\usepackage{wrapfig}
\usepackage{bbm}
\usepackage{mathrsfs}
\usepackage{amsmath,amssymb}
\usepackage[round,semicolon]{natbib}

\usepackage[final,pagebackref=true]{hyperref} 
\renewcommand*{\backrefalt}[4]{%
    \ifcase #1 \footnotesize{(Not cited.)}%
    \or        \footnotesize{(Cited on page~#2)}%
    \else      \footnotesize{(Cited on pages~#2)}%
    \fi}

\hypersetup{
	colorlinks=true,       
	linkcolor=blue,        
	citecolor=blue,        
	filecolor=magenta,     
	urlcolor=blue         
}

\usepackage{mathtools}
\usepackage{amsthm}
\usepackage{amsmath,amsfonts,bm}
\makeatletter
\newtheorem*{rep@theorem}{\rep@title}
\newcommand{\newreptheorem}[2]{%
\newenvironment{rep#1}[1]{%
 \def\rep@title{#2 \ref{##1}}%
 \begin{rep@theorem}}%
 {\end{rep@theorem}}}
\makeatother

\newtheorem{definition}{Definition}[section]
\newtheorem{theorem}{Theorem}
\newtheorem{remark}{Remark}
\newtheorem{proposition}{Proposition}

\newtheorem{assumption}{Assumption}
\newcommand{\aref}[1]{\hyperref[#1]{A\ref*{#1}}}

\newtheorem{lemma}{Lemma}

\newenvironment{proofsketch}{%
  \proof}{\endproof}

\usepackage[framemethod=TikZ]{mdframed}
\mdfdefinestyle{MyFrame}{%
    linecolor=black,
    outerlinewidth=.3pt,
    roundcorner=5pt,
    innertopmargin=1pt, 
    innerbottommargin=1pt, 
    innerrightmargin=1pt,
    innerleftmargin=1pt,
    backgroundcolor=black!0!white}

\mdfdefinestyle{MyFrame2}{%
    linecolor=white,
    outerlinewidth=1pt,
    roundcorner=2pt,
    innertopmargin=\baselineskip,
    innerbottommargin=\baselineskip,
    innerrightmargin=10pt,
    innerleftmargin=10pt,
    backgroundcolor=black!3!white}

\newcommand{\RNum}[1]{\uppercase\expandafter{\romannumeral #1\relax}}

\newcommand{\vertiii}[1]{{\left\vert\kern-0.25ex\left\vert\kern-0.25ex\left\vert #1 
    \right\vert\kern-0.25ex\right\vert\kern-0.25ex\right\vert}}
\newcommand{\vertiiii}[1]{{\vert\kern-0.25ex\vert\kern-0.25ex\vert #1 
    \vert\kern-0.25ex\vert\kern-0.25ex\vert}}

\usepackage{mathtools}

\DeclareMathOperator*{\argmin}{\arg\!\min}

\newcommand{\cut}[1]{}

\newcommand{\removelatexerror}{\let\@latex@error\@gobble}

\def\1{\bm{1}}

\def\eps{{\epsilon}}

\DeclareMathAlphabet{\mathsfit}{\encodingdefault}{\sfdefault}{m}{sl}
\SetMathAlphabet{\mathsfit}{bold}{\encodingdefault}{\sfdefault}{bx}{n}

\begin{document}

\twocolumn[

\aistatstitle{Performative Prediction with Neural Networks}

\aistatsauthor{ Mehrnaz Mofakhami \And Ioannis Mitliagkas \And  Gauthier Gidel }

\aistatsaddress{ Mila, Université de Montréal \And  Mila, Université de Montréal \\ Canada CIFAR AI Chair \And Mila, Université de Montréal  \\ Canada CIFAR AI Chair} ]

\begin{abstract}
Performative prediction is a framework for learning models that influence the data they intend to predict. We focus on finding classifiers that are \textit{performatively stable}, i.e.\ optimal for the data distribution they induce. Standard convergence results for finding a performatively stable classifier with the method of repeated risk minimization assume that the data distribution is Lipschitz continuous to the {\em model's parameters}.
Under this assumption, the loss {\em must} be strongly convex and smooth in these parameters; otherwise, the method will diverge for some problems. 
In this work, we instead assume that the data distribution is Lipschitz continuous with respect to the \textit{model's predictions}, a more natural assumption for performative systems. As a result, we are able to relax the assumptions on the loss function significantly. In particular, we do not need to assume convexity with respect to the model's parameters. As an illustration, we introduce a resampling procedure that models realistic distribution shifts and show that it satisfies our assumptions.
We support our theory by showing that one can learn performatively stable classifiers with neural networks making predictions about real data that shift according to our proposed procedure.
\end{abstract}

\looseness=-1 
\section{INTRODUCTION}
One of the main challenges in many of the decision-making tasks is that the data distribution changes over time. This concept is known as distribution shift or concept drift \citep{gamma2014survey, Tsymbal2004problem, Quionero2009dataset}. With a changing data distribution, the performance of a supervised learning classifier may degrade since it implicitly assumes a static relationship between input and output variables \citep{fang2020rethinking}. \textit{Performative prediction} is a framework introduced by \citet{performative} to deal with this problem when the distribution changes as a consequence of model deployment, usually through actions taken based on the model's predictions. For example, election predictions affect campaign activities and, in turn, influence the final election results. This performative behavior arises naturally in many problems of economics, social sciences, and machine learning, such as in loan granting, predictive policing, and recommender systems \citep{performative, Krauth2022breaking, ensign2018runaway}.

So far, most works in this area assume strong convexity of the risk function $\theta \mapsto \ell(z;\theta)$, which takes as input the model's parameters $\theta$, and a data point $z$ \citep{performative, mendler2020stochastic, brown2022stateful}. For example, \citet{performative} show that by assuming strong convexity and smoothness of the loss function along with some regularity assumptions on the data generation process, repeated retraining converges to a performative stable classifier, which is a model that is optimal for the distribution it induces. However, this strong convexity assumption does not hold for most modern ML models, e.g.\ neural networks. 
From a different perspective, given a data point $(x,y)$, the risk function can be expressed as a mapping from the prediction $x \mapsto f_\theta(x)$ to a loss between the prediction $\hat y:= f_\theta(x)$ and the target $y$, in which case convexity almost always holds. Consider, for example, the Squared Error loss function in a binary classification problem where $\ell(f_{\theta}(x),y) = \frac{1}{2}(f_{\theta}(x) - y)^2$ for a data point $z=(x,y)$; it is convex with respect to the model's predictions $f_{\theta}(x)$, but not necessarily with respect to the model's parameters $\theta$.

With this in mind, we propose a different perspective and formulation that shifts attention from the space of parameters to the space of predictions. More precisely, we require distributions to be functions of the model's prediction function instead of its parameters. The rationale behind this is that in many scenarios with performative effects of a classifier in the loop, the model's predictions are the quantities of interest rather than its parameters. Actually, the framework assumes the data distribution changes as a result of model deployment, and at the time of deployment, it is the final predictions that matter rather than the parameters that led to those predictions. 

Within our formulation, we show that by having a stronger assumption on the distribution map than the original framework, 
we can relax the convexity condition on the loss function and prove the existence and uniqueness of a performative stable classifier under repeated risk minimization with significantly weaker assumptions with regard to the regularity of the loss. Informally, these assumptions include strong convexity of the loss with respect to the predictions, the boundedness of the derivative of the loss function, and the Lipschitzness of the distribution map with respect to the $\chi^2$ divergence. This more general set of assumptions on the loss function lets us theoretically analyze the performative effects of neural networks with non-convex loss functions; we believe this is a significant step toward bridging the gap between the theoretical performative prediction framework and realistic settings. Our setting and the main theoretical results will be explained in Section~\ref{sec2}.

\subsection{Background}
Before stating our main theoretical contribution, we first recall the key concepts of the performative prediction framework. The main difference between classic supervised learning and the performative prediction framework is that the latter considers the data distribution to be model-dependent, i.e. it assumes that the distribution map directly depends on the model's parameters $\theta$ and is denoted by $\theta \mapsto \mathcal{D}(\theta)$. The distribution map is said to satisfy a notion of Lipschitz continuity called \textit{$\epsilon$-sensitivity} if for any $\theta$ and $\theta'$, $\mathcal{W}_1\left(\mathcal{D}(\theta), \mathcal{D}(\theta')\right)\le \epsilon \|\theta - \theta'\|_2$, where $\mathcal{W}_1$ denotes the Wasserstein-1 distance.
The performance of a model with parameters $\theta$ is measured by its $\textit{performative risk}$ under the loss function $\ell$ which is stated as a function of a data point $z$ and $\theta$:
\begin{equation*}
    PR(\theta) \overset{\text{def}}{=} \mathop{\mathbb{E}}_{z\sim \mathcal{D}(\theta)} \ell(z;\theta).
\end{equation*}
A classifier with parameters $\theta_{\text{PS}}$ is \textit{performatively stable} if it minimizes the risk on the distribution it induces. In other words, it is the fixed point of repeated retraining.  
\begin{equation*}
    \theta_{\text{PS}} = \argmin_{\theta \in \Theta} \mathop{\mathbb{E}}_{z\sim \mathcal{D}(\theta_{\text{PS}})} \ell(z;\theta).
\end{equation*}
\citet{performative} demonstrate that with an $\epsilon$-sensitive distribution map, $\gamma$-strong convexity of $\ell$ in $\theta$ and \mbox{$\beta$-smoothness} of $\ell$ in $\theta$ and $z$ are sufficient and \emph{necessary} conditions for repeated risk minimization to converge to a performatively stable classifier if $\frac{\epsilon\beta}{\gamma}<1$. We show, however, that by slightly changing the assumptions on the distribution map, we can break their negative result regarding the necessary strong convexity and show that one can converge even when the loss is non-convex in $\theta$.

\subsection{Our contributions} \label{contribution}
   Our paper provides sufficient conditions for the convergence of repeated risk minimization to a classifier with unique predictions under performative effects in the absence of convexity to the model's parameters. The key idea in our framework is that the distribution map is no longer a function of the parameters $\theta$, but a function of the model's predictions, denoted by $\mathcal{D}(f_{\theta})$ where $f_{\theta}$ is in $\mathcal{F}$, the set of parameterized functions by \(\theta\in \Theta\). We also express the loss $\ell$ as a function of the prediction $f_{\theta}(x)$ and the target $y$, where both can be multi-dimensional. Following is the informal statement of our main theorem.
   
\begin{theorem}(Informal)\label{informal} If the loss $\ell(f_{\theta}(x),y)$ is strongly convex in $f_{\theta}(x)$ with a bounded gradient norm, and the distribution map $f_{\theta} \mapsto \mathcal{D}(f_{\theta})$ is sufficiently Lipschitz with respect to the $\chi^2$ divergence and satisfies a bounded norm ratio condition, then repeated risk minimization converges linearly to a stable classifier with unique predictions if the function space is convex and to a small neighborhood around a unique stable classifier if it is not.
\end{theorem}

We will state this theorem formally in Section~\ref{sec2}. The critical assumption we make on the distribution map is Lipschitz continuity, which captures the idea that a small change in the model's predictions cannot lead to a large change in the induced data distribution, as measured by the $\chi^2$ divergence.
This is more restrictive than the Lipschitz continuity assumption of \citet{performative} with $\mathcal{W}_1$, since for the $\chi^2$ divergence to be finite, distributions should have the same support. However, we show that this still holds in realistic settings, and we discuss that this stronger assumption on the distribution map is a price we have to pay to relax the assumptions on the loss function significantly and have convergence guarantees for neural networks with non-convex loss functions.

In Section~\ref{experiments}, we demonstrate our main results empirically with a \textit{strategic classification} task, which has been used as a benchmark for performative prediction \citep{performative, miller2021outside, brown2022stateful}. Strategic classification involves an institution that deploys a classifier and agents who strategically manipulate their features to alter the classifier's
predictions to get better outcomes. 
We propose a resampling procedure called \textit{Resample-if-Rejected (RIR)} in Section~\ref{RIR} to model the population's strategic responses and show that it results in a distribution map that satisfies the conditions of Theorem~\ref{informal}.

Within this process, a sample $x$ is drawn from the base distribution and is rejected with some probability dependent on $f_{\theta}(x)$ and accepted otherwise; in case of rejection, another sample from the base distribution will be drawn.

A real-life example that this procedure may be able to model is regarding posting content on social media. Actually, social media use many ML models to automatically regulate the content posted by the users. Consequently, some users' posts may be rejected by the automatic regulation because their content was considered to violate the platform's community guidelines. In some situations, the authors might consider this rejection unfair and may tweak some parts of the post in order to get accepted. In our experiments, we model this resubmission by resampling some strategic features (i.e., features that do not drastically affect the content and that are easily modifiable) of the post. 

\subsection{Related work}
Prior work on performative prediction focused on learning from a data distribution $\mathcal{D}(\theta)$ that could change with the model's parameter $\theta$ \citep{performative,mendler2020stochastic,brown2022stateful,drusvyatskiy2022stochastic,miller2021outside,izzo2021learn,maheshwari2022zeroth,ray2022decision,li2022state,dong2021approximate,jagadeesan2022regret}. In this work, we propose to strengthen the standard $\epsilon$-sensitivity assumption on the distribution map initially proposed by~\citet{performative}. To a certain extent, we propose a novel $\epsilon$-sensitivity assumption for the performative prediction framework that allows us to relax the convexity assumption on the loss function. Such relaxation is essential if we want to consider the practical setting of classifiers \mbox{parametrized by neural networks.} 

At a technical level, our analysis is inspired by~\citet[Theorem ~3.5]{performative}. However, because our quantity of interest is a distance in the function space (see Theorem~\ref{thm}) our proof significantly differs from~\citet{performative}. We require a different notion of $\epsilon$-sensitivity and an additional assumption (Assumption~\ref{Assumption2}) in order to control the variation of the functional norm defined in Assumption~\ref{A1}.  

Various prior works have focused on finding performatively stable classifiers \citep{performative, brown2022stateful, mendler2020stochastic, li2022state}, but to the best of our knowledge, none of them analyze the convergence of repeated retraining with loss functions that might be non-convex to the model's parameters.

Exploiting convexity in the model's predictions has previously been explored by ~\citet{bengio2005convex}, who noticed that most of the loss functions to train neural networks are convex with respect to the neural network itself. There have been many works trying to leverage this property to show convergence results applied to neural networks in the context of machine learning~\citep{bach2017breaking,chizat2018global,mladenovic2021generalized}. However, none of these results are in the context of performative prediction. \citet{jagadeesan2022regret} proposes an algorithm to find classifiers with near-optimal performative risk without assuming convexity. First, their work focuses on a different notion of optimality (namely, performatively optimal points). Second, they focus on regret minimization, while our work is concerned with finding a performatively stable classifier with gradient-based algorithms and having guarantees to make sure we converge to such a stable classifier within a reasonable \mbox{number of steps.}\footnote{For a $\delta$-approximate optimum, \citet{jagadeesan2022regret} propose an algorithm that requires $O(1/\delta^{d})$ repeated minimizations for the last iterate where $d$ is some notion of dimension. In comparison, in Theorem~\ref{thm} we require $O(\log(1/\delta))$ minimizations.}

Similarly to our work, \citet{mendler-dünner2022predicting} assume that the performativity of a model occurs through its predictions and consider the distribution a function of the predictive model. However, this paper has a different focus entirely; they try to find a set of conditions under which the causal effect of predictions becomes identifiable. Additionally, they focus on a subset of performative prediction problems where predictions only influence the target variable and not the features $X$. Hence, their analysis does not capture strategic classification.

\section{FRAMEWORK AND MAIN RESULTS}\label{sec2}

To propose our main theorem, we first need to redefine some of the existing concepts. As mentioned earlier, we assume $\mathcal{D}(\theta)$ to be a mapping from the model's prediction function $f_{\theta}$ to a distribution $\mathcal{D}(f_{\theta})$ over instances $z$, where $f_{\theta}$ is in $\mathcal{F}$, the set of parameterized neural networks by $\Theta$. Each instance $z$ is a pair of features and label $(x,y)$. With this new formulation, the objective risk function will be defined as follows:
\begin{definition} [\textit{Performative Risk}] \label{def2}
 Performative risk (PR) is defined as follows:
\[\text{PR}(f_{\theta}) \overset{\text{def}}{=} \mathop{\mathbb{E}}_{z\sim \mathcal{D}(f_{\theta})} \ell(f_{\theta}(x),y).\]
\end{definition}

In this work, we focus on finding a performatively stable classifier, which minimizes the risk on the distribution its prediction function entails:
\begin{definition} A classifier with parameters $\theta_{\text{PS}}$ is performatively stable if:\\
\[\theta_{\text{PS}} = \argmin_{\theta \in \Theta} \mathop{\mathbb{E}}_{z\sim \mathcal{D}(f_{\theta_{\text{PS}}})} \ell(f_{\theta}(x),y).\]
\end{definition}
Repeated retraining is the algorithm we use to find a stable classifier, which is defined formally as follows:
\begin{definition}[\textit{RRM}] \label{def3}
Repeated Risk Minimization (RRM) refers to the procedure where, starting from an initial \(\theta_0\), we perform the following sequence of updates for every \(t\ge 0\):
\[\theta_{t+1}=G(\theta_t)\overset{\text{def}}{=} \argmin_{\theta\in \Theta} \mathop{\mathbb{E}}_{z\sim \mathcal{D}(f_{\theta_t})} \ell(f_{\theta}(x),y).\]
\end{definition}

\subsection{Assumptions}
In order to provide convergence guarantees for repeated retraining, we require regularity assumptions on the distribution map and the loss function. A natural assumption we make on $\mathcal{D}(.)$ inspired by prior work is Lipschitz continuity, formally referred to as \textit{$\epsilon$-sensitivity}. Intuitively, this assumption states the idea that
if two models with similar prediction functions are deployed, then the induced distributions should also be similar. We refer to the \emph{base distribution} $\mathcal{D}$ as the distribution over (features, label) pairs before any classifier deployment.
\begin{assumption}\label{A1} (A1) [\textit{$\epsilon$-sensitivity w.r.t Pearson $\chi^2$ divergence}] Suppose the base distribution $\mathcal{D}$ has the probability density function (pdf) $p$ over instances $z=(x,y)$.
The distribution map $\mathcal{D}(.)$ which maps $f_{\theta}$ 
to $\mathcal{D}(f_{\theta})$ with the pdf $p_{f_{\theta}}$ is \(\epsilon\)-sensitive w.r.t  Pearson $\chi^2$ divergence, i.e., for all  \(f_{\theta}\) and \(f_{\theta'}\) in \(\mathcal{F}\) the following holds:
\begin{equation*}
\chi^2(\mathcal{D}(f_{\theta'}), \mathcal{D}(f_{\theta})) \le \epsilon \|f_{\theta}-f_{\theta'}\|^2,
\end{equation*}
where $\|f_{\theta} - f_{\theta'}\|^2 := \int \|f_{\theta}(x) - f_{\theta'}(x)\|^2 p(x) dx$
and $\chi^2(\mathcal{D}(f_{\theta'}), \mathcal{D}(f_{\theta})) := \int \frac{\left(p_{f_{\theta'}}(z) - p_{f_{\theta}}(z)\right)^2}{p_{f_{\theta}}(z)} dz$
\end{assumption}

\begin{assumption}\label{Assumption2} (A2,A2$^\prime$) [\textit{Bounded norm ratio}]
     The distribution map $\mathcal{D}(.)$ satisfies bounded norm ratio with the parameters \mbox{$C\ge1$}, and $c\le C$ if for all $f_{\theta}, f_{\theta'}, f_{\theta^*} \in\mathcal{F}$:
     \begin{equation}\|f_{\theta}-f_{\theta'}\|^2 \le C \|f_{\theta}-f_{\theta'}\|_{f_{\theta^*}}^2, \tag{$A2$} \label{A2}
     \end{equation}
     and
     \begin{equation}
         c \|f_{\theta}-f_{\theta'}\|_{f_{\theta^*}}^2 \le\|f_{\theta}-f_{\theta'}\|^2, \tag{$A2'$}\label{A2prime}
     \end{equation}
     where 
     \[\|f_{\theta}-f_{\theta'}\|_{f_{\theta^*}}^2 = \int \|f_{\theta}(x) - f_{\theta'}(x)\|^2 p_{f_{\theta^*}}(x) dx\]
     is a notation for an $f_{\theta^*}$-dependent norm. In other words, the above inequalities will be equivalent to
         \[\mathbb{E}_{p(x)} [\|f_{\theta}(x)-f_{\theta'}(x)\|^2]
    \le C\  \mathbb{E}_{p_{f_{\theta^*}}(x)} [\|f_{\theta}(x)-f_{\theta'}(x)\|^2],\]
    and
     \[c\ \mathbb{E}_{p_{f_{\theta^*}}(x)} [\|f_{\theta}(x)-f_{\theta'}(x)\|^2] \le \mathbb{E}_{p(x)} [\|f_{\theta}(x)-f_{\theta'}(x)\|^2],\]
    for $(A2)$ and $(A2^\prime)$ respectively, where $p(x)$ and $p_{f_{\theta^*}}(x)$ are pdfs for the marginal distribution of $X$ according to $\mathcal{D}$ and $\mathcal{D}(f_{\theta^*})$.
\end{assumption}

The distribution map satisfies the bounded norm ratio condition if the bounded density ratio property holds, i.e. $c\ p_{f_{\theta}(x)}\le p(x)\le C\ p_{f_{\theta}}(x)$ for every $f_{\theta}\in \mathcal{F}$. We will show how the bounded density ratio holds in our example in Section~\ref{RIR}.

Our notion of Lipschitz Continuity uses the Pearson $\chi^2$ divergence---interchangeably referred to as $\chi^2$ divergence---to measure the distance between distributions, as opposed to \citet{performative} who use $\mathcal{W}_1$ distance. Using $\chi^2$ divergence is more restrictive since the distributions should have the same support for the $\chi^2$ divergence between them to be finite.

As stated in Remark~\ref{remark:W1}, $\epsilon$-sensitivity with respect to $\chi^2$ implies $K\sqrt{\epsilon}$-sensitivity with respect to $\mathcal{W}_1$ for a constant $K$ that depends on the diameter of the space $d_{max}$. If $\frac{d_{max}}{2}<\sqrt{\epsilon}$, our notion of $\epsilon$-sensitivity w.r.t $\chi^2$ is indeed stronger than the corresponding notion of $\epsilon$-sensitivity w.r.t $\mathcal{W}_1$. However, Proposition~\ref{proposition1} explains why we cannot replace $\chi^2$ with $\mathcal{W}_1$ within our settings.

\begin{remark} \label{remark:W1}
    For two distributions $\mathcal{D}(x)$ and $\mathcal{D}'(x)$, the $\mathcal{W}_1$ distance is upper bounded by a coefficient of the square root of $\chi^2$ divergence \citep[Figure 8.2]{peyre2019computational}:
    \begin{equation*}
        \mathcal{W}_1(\mathcal{D}(x), \mathcal{D}'(x)) \le \frac{d_{max}}{2}\sqrt{\chi^2(\mathcal{D}(x), \mathcal{D}'(x))},
    \end{equation*}
    where $\mathcal{X}$ is a metric space with ground distance $d$ and $d_{max} = \sup_{(x,x')} d(x,x')$ is the diameter of $\mathcal{X}$. \\
    If we define $\epsilon$-sensitivity of $\mathcal{D}(.)$ w.r.t $\mathcal{W}_1$ as
\begin{equation*} \label{eq:A1prime}
        \mathcal{W}_1(\mathcal{D}(f_{\theta}), \mathcal{D}(f_{\theta'})) \le \epsilon \|f_{\theta} - f_{\theta'}\|, \tag*{${(A1)^\prime}$}   
\end{equation*}
    then $\epsilon$-sensitivity w.r.t $\chi^2$ implies $\frac{d_{max}}{2}\sqrt{\epsilon}$-sensitivity with respect to $\mathcal{W}_1$:
    \begin{equation*}
        \mathcal{W}_1(\mathcal{D}(f_{\theta}), \mathcal{D}(f_{\theta'})) \le \frac{d_{max}}{2}\sqrt{\epsilon} \|f_{\theta} - f_{\theta'}\|.
    \end{equation*}
\end{remark}

Despite the downsides of our assumptions on the distribution map, these assumptions still hold in some realistic settings, an example of which is a resampling procedure proposed in Section~\ref{RIR}. The idea of this distribution shift is that individuals are more likely to change their features (by resampling them) if there is a high chance that they will receive an unfavorable classification outcome, which is quantified by the model's prediction. 

While imposing a more restrictive assumption on the distribution map, we significantly relax the assumptions on the loss function. In particular, we no longer need to assume that loss is convex to the model's parameters, which opens the door to consider deep neural networks as classifiers in our analysis. We still require some mild assumptions on the loss function $\ell$ that are as follows:

\begin{assumption}\label{A3} (A3) [\textit{Strong convexity w.r.t predictions}] The loss function $\ell(f_{\theta}(x),y)$ which takes as inputs the prediction $f_{\theta}(x)$ and the target $y$, is $\gamma$-strongly convex in $f_{\theta}(x)$. More precisely, the following inequality holds for every $f_{\theta},f_{\theta'} \in \mathcal{F}$:
\begin{align*}
\ell(f_{\theta}(x),y) & \ge \ell(f_{\theta'}(x),y) +  \\ & (f_{\theta}(x) - f_{\theta'} (x))^\top \nabla_{\hat{y}} \ell(f_{\theta'}(x),y) +
\\& \frac{\gamma}{2} \| f_{\theta}(x) - f_{\theta'}(x)\|^2,  
\end{align*}
where $\nabla_{\hat y} \ell(f_{\theta}(x) ,y)$ is the gradient of the function \mbox{$\hat y \in \mathbb{R}^d \mapsto \ell(\hat y,y)$} at $f_{\theta}(x)$.
\end{assumption}

\begin{assumption}\label{A4} (A4)  [\textit{Bounded gradient norm}] The loss function $\ell(f_{\theta}(x),y)$ has bounded gradient norm, i.e., the norm of its gradient with respect to \(f_{\theta}(x)\) is upper bounded with a finite value \(M = \sup_{x,y,\theta} \|\nabla_{\hat{y}}\ell(f_{\theta}(x),y)\|\).
\end{assumption}

We can easily see that these two assumptions on $\ell$ are satisfied by the Squared Error loss: \mbox{$\ell(f_{\theta}(x), y) = \frac{1}{2}\|f_{\theta}(x) - y\|^2$}. This function is $1$-strongly convex with a bounded gradient norm of $\sqrt{d}$ if $y$ is a one-hot vector in $\mathbb{R}^d$ and $f_{\theta}(x)\in [0,1]^d$ for any $\theta$. More broadly, when the predictions are bounded, e.g. in $[0,1]^d$, then the quantity $M$ in Assumption~\ref{A4} always exists for continuously differentiable loss functions, which makes it a very mild assumption.

In Section~\ref{Sec2.2} we show that if assumptions $A1-A4$ are satisfied for a distribution map and a loss function, then RRM on a convex $\mathcal{F}$ converges to a unique stable classifier if $\frac{\sqrt{C\epsilon}M}{\gamma}$ is less than $1$.
Proposition~\ref{proposition1} shows that replacing $\epsilon$-sensitivity w.r.t $\chi^2$ (\aref{A1}) with $\epsilon$-sensitivity w.r.t $\mathcal{W}_1$~\ref{eq:A1prime} while keeping other assumptions would break this convergence result, in the sense that RRM will oscillate between two models forever. This justifies why we cannot use $\mathcal{W}_1$ within our analysis.
\vspace{3pt}
\begin{proposition} \label{proposition1}
    Suppose that the loss $\ell(f_{\theta}(x),y)$ is \mbox{$\gamma$-strongly} convex in $f_{\theta}(x)$, and has a derivative bounded by $M$. If the distribution map satisfies the bounded norm ratio property with a parameter $C$, and it is $\epsilon$-sensitive w.r.t $\mathcal{W}_1$~\ref{eq:A1prime}, RRM may diverge for any value of $\epsilon$, particularly even if $\frac{\sqrt{C\epsilon}M}{\gamma}<1.$
\end{proposition}
\begin{proof}\label{proof:counter_example}
Consider a supervised learning problem where a model with parameters $\theta$ uses the prediction function $f_{\theta}(x) = \frac{\tanh(\theta)+2}{\epsilon}x$ where $x\in(0,3\epsilon]$. Take the base distribution on $X$ as a uniform distribution \mbox{over this interval.}\\
The loss function is defined as
\[
    \ell(f_{\theta}(x),y) = \frac{-15\gamma}{4} (f_{\theta}(x) - y ) + \frac{\gamma}{2} f_{\theta}(x)^2 + \frac{\gamma}{2} y^2 + \gamma(\frac{15}{4})^2.
\]
This $\ell$ is non-negative, $\gamma$-strongly convex w.r.t $f_{\theta}(x)$, and its derivative in $f_{\theta}(x)$ is bounded.

Let the distribution of $X$ according to $\mathcal{D}(f_{\theta})$ be a point mass at $f_\theta(\epsilon^2)=\epsilon(\tanh(\theta)+2)$ and the distribution of $Y$ be invariant w.r.t $f_{\theta}$. 

$\mathcal{D}(f_{\theta})$ is $\epsilon$-sensitive w.r.t the Wasserstein-1 distance:

Choose $f_{\theta}$ and $f_{\theta'}$ arbitrarily. It is easy to see that 
\begin{equation} \label{exeq:1}
    \mathcal{W}_1(\mathcal{D}(f_{\theta}), \mathcal{D}(f_{\theta'})) \le \epsilon |\tanh(\theta) - \tanh(\theta')|.
\end{equation}
\begin{align}
        \|f_{\theta} - f_{\theta'}\|^2 & = \int_{0}^{3\epsilon} (f_{\theta}(x) - f_{\theta'}(x))^2 p(x) dx \notag\\
        & = \int_{0}^{3\epsilon} \frac{\left(\tanh(\theta) - \tanh(\theta')\right)^2}{\epsilon^2} x^2 p(x) dx \notag\\
        & = \frac{\left(\tanh(\theta) - \tanh(\theta')\right)^2}{\epsilon^2}  \frac{1}{3\epsilon} \int_{0}^{3\epsilon} x^2 dx\notag \\
        & = \frac{\left(\tanh(\theta) - \tanh(\theta')\right)^2}{3\epsilon^3} \frac{(3\epsilon)^3}{3} \notag \\
        & = 3 \left(\tanh(\theta) - \tanh(\theta')\right)^2. \label{exeq:2}
\end{align}
As a result,
\begin{equation} \label{exeq:3}
    \|f_{\theta} - f_{\theta'}\| = \sqrt{3} |\tanh(\theta) - \tanh(\theta')|.
\end{equation}
Combining ~(\ref{exeq:1}) and ~(\ref{exeq:3}) results in the $\epsilon$-sensitivity:
\begin{equation*}
    \mathcal{W}_1(\mathcal{D}(f_{\theta}), \mathcal{D}(f_{\theta'})) \le \epsilon \|f_{\theta} - f_{\theta'}\|.
\end{equation*}

Also, this distribution map satisfies the bounded norm ratio property with any $C>3$ since:
\begin{align} \label{exeq:4}
    & \|f_{\theta}-f_{\theta'}\|_{f_{\theta^*}}^2 \notag \\ & = \left(f_{\theta}(\epsilon(\tanh(\theta^*)+2)) - f_{\theta'}(\epsilon(\tanh(\theta^*)+2))\right)^2 \notag \\
    & = \big((\tanh(\theta) - \tanh(\theta'))(\tanh(\theta^*)+2)\big)^2 \notag \\
    & > (\tanh(\theta) - \tanh(\theta'))^2,
\end{align}
where we used the fact that $(\tanh(\theta^*)+2)>1$.

Putting ~(\ref{exeq:2}) and ~(\ref{exeq:4}) together, we can write
\begin{equation*}
    \|f_{\theta} - f_{\theta'}\|^2 \le C\  \|f_{\theta}-f_{\theta'}\|_{f_{\theta^*}}^2,
\end{equation*}
for any $C>3$.\\
To verify that the space of functions is convex, we need to show that for any $\alpha \in [0,1]$, and any $\theta,\theta'$, there exists $\theta''$ such that $\alpha f_{\theta}(x) + (1-\alpha)f_{\theta'}(x) = f_{\theta''}(x)$, \mbox{$\forall x \in (0,3\epsilon]$}. To prove this, it suffices to show that $\theta''$ exists such that $\alpha \tanh(\theta) + (1-\alpha) \tanh(\theta') = \tanh(\theta'')$. This follows from the continuity of the $\tanh$ function.

The update rule of RRM is as follows:
\begin{align*}
    \theta_{t+1} & = \argmin_{\phi} \mathbb{E}_{z\sim \mathcal{D}(f_{\theta_t})} [\ell(f_{\phi}(x),y)]  \\ & = \argmin_{\phi} \ell(f_{\phi}(x),y)\bigg|_{x = \epsilon (\tanh(\theta_t)+2)}.
\end{align*}
Taking the derivative of the loss and setting it to zero results in:
\begin{equation*}
(\tanh(\theta_{t+1})+2)(\tanh(\theta_t)+2) = \frac{15}{4}.
\end{equation*}
So, if $\theta_t=\tanh^{-1}(\frac{-1}{2})$, then $\theta_{t+1} = \tanh^{-1}(\frac{1}{2})$ and if $\theta_t=\tanh^{-1}(\frac{1}{2})$, then $\theta_{t+1} = \tanh^{-1}(\frac{-1}{2})$ . 

In conclusion, while the space of functions is convex, the loss function satisfies assumptions (\aref{A3}) and (\aref{A4}) and the distribution map satisfies conditions \ref{eq:A1prime} and \eqref{A2}, RRM oscillates between $\tanh^{-1}(\frac{-1}{2})$ and $\tanh^{-1}(\frac{1}{2})$ with \mbox{$\theta_0 = \tanh^{-1}(\frac{-1}{2})$}, for any value of $\epsilon, \gamma, C>3$, including when $\frac{\sqrt{C\epsilon}M}{\gamma}<1$.
\end{proof}

\subsection{Convergence of RRM} \label{Sec2.2}
Here we state our main theoretical contribution, which provides sufficient conditions for repeated risk minimization to converge to a stable classifier with unique predictions. The theorem is in two parts. The first part provides convergence guarantees to a stable classifier when the class of functions $\mathcal{F}$ is convex. The second part shows that when the class of functions is not convex, we can converge to a neighborhood of a stable classifier that can get arbitrarily small based on an approximation error $\kappa$.
\begin{theorem}\label{thm}
Suppose that the loss \(\ell(f_{\theta}(x),y)\) is \mbox{$\gamma$-strongly} convex and continuously differentiable w.r.t \(f_{\theta}(x)\) (\aref{A3}) and the norm of its gradient w.r.t \(f_{\theta}(x)\) is upper bounded with \mbox{$M = \sup_{x,y,\theta} \|\nabla_{\hat{y}}\ell(f_{\theta}(x),y)\|$} (\aref{A4}). If the distribution map \(\mathcal{D}(.)\) is \(\epsilon\)-sensitive w.r.t Pearson $\chi^2$ divergence (\aref{A1}) and satisfies the bounded norm ratio property with parameter $C$ \eqref{A2}, then:

(a) when the function space $\mathcal{F}$ is convex, i.e., for any \mbox{$f, f' \in \mathcal{F}$} and $\alpha \in [0,1]$, $\alpha f + (1-\alpha) f'\in \mathcal{F}$,
\[ \|f_{G(\theta)} - f_{G(\theta')}\| \le \frac{\sqrt{C\epsilon}M}{\gamma} \|f_{\theta}-f_{\theta'}\|.\]
So if \( \frac{\sqrt{C\epsilon}M}{\gamma}<1\), $G$ is a contractive mapping and RRM converges to a unique stable classifier $f_{\theta_{\text{PS}}}$ at a linear rate: \[ \|f_{\theta_t} - f_{\theta_{\text{PS}}}\| \le (\frac{\sqrt{C\epsilon}M}{\gamma})^t \|f_{\theta_0}-f_{\theta_{\text{PS}}}\|\]

(b) when the function space $\mathcal{F}$ is not necessarily convex but can approximate $f^*_{\theta}(x) = \mathbb{E}_{p_{f_{\theta}}} [Y|X=x]$ for any $\theta$ with an error of at most $\kappa$, i.e., $\exists f\in \mathcal{F}$ s.t. $\|f-f^*_{\theta}\|\le \kappa$,\footnote{We also need to assume that we can approximate $f^* = \argmin\limits_{f\in \mathcal{M}} \mathbb{E}_{z\sim \mathcal{D}(f^*)} [\ell(f(x),y)],$ with at error of at most $\kappa$, where $\mathcal{M}$ is the set of all measurable functions. Since $\mathcal{M}$ is convex, we get from part (a) of the theorem that $f^*$ exists and is unique.} and the distribution map also satisfies the bounded norm ratio with parameter $c$~\eqref{A2prime}, applying RRM with squared error loss $\ell(f(x),y) = \frac{1}{2}\|f(x)-y\|^2$ will result in the following:
\begin{align*}
      \|f_{\theta_{t+1}} - f^*\| & \le \frac{2\sqrt{C \epsilon}M}{\gamma} \|f_{\theta_t} - f^*\| \notag \\ 
      & + \big(\frac{3\sqrt{C \epsilon}M}{\gamma} + \frac{2C}{\gamma c} + \sqrt{\frac{C}{c}}\big)\kappa.
\end{align*}
$f^* = \argmin\limits_{f\in \mathcal{M}} \mathbb{E}_{z\sim \mathcal{D}(f^*)} [\ell(f(x),y)],$ is the unique stable classifier under $\mathcal{M}$, the set of all measurable functions.\\
This implies that, if $\frac{2\sqrt{C \epsilon}M}{\gamma}<1$, RRM converges to a neighborhood around the stable solution at a linear rate where the size of this neighborhood depends on $\kappa$ and diminishes as $\kappa$ decreases: 
\begin{align*}
      \|f_{\theta_{t}} - f^*\| & \le \left(\frac{2\sqrt{C \epsilon}M}{\gamma}\right)^t \|f_{\theta_0} - f^*\| + O(\kappa).
\end{align*}


\end{theorem}

The assumption in part (b) is reasonable for a class of large or deep enough Neural Networks based on the fact that they are universal approximators as we explain in Remark \ref{remark:sobolev}.
Also as we mentioned earlier, assumptions (\aref{A3}) and (\aref{A4}) on $\ell$ are satisfied by the commonly-used Squared Error loss function, and this holds even in the presence of deep neural networks as predictors. To illustrate our results, we propose the \textit{Resample-if-Rejected} procedure in the following section and show that it satisfies assumptions (\aref{A1}), \eqref{A2}, and \eqref{A2prime}. We provide a proof sketch for Theorem~\ref{thm} part (a) here. The full proof of the theorem is available in the Supplementary Materials~\ref{proof_main_thm}.

\begin{proofsketch}
Fix \(\theta\) and \(\theta'\) in \(\Theta\). Let \(h\) and \(h'\) be mappings from $\mathcal{F}$ to $\mathbb{R}$ defined as follows: 
\begin{equation*}
h(f_{\hat{\theta}}) = E_{z\sim \mathcal{D}(f_{\theta})} [\ell(f_{\hat{\theta}}(x),y)] = \int \ell(f_{\hat{\theta}}(x),y) p_{f_{\theta}}(z) dz.
\end{equation*}
\begin{equation*}
h'(f_{\hat{\theta}}) = E_{z\sim \mathcal{D}(f_{\theta'})} [\ell(f_{\hat{\theta}}(x),y)] = \int \ell(f_{\hat{\theta}}(x),y) p_{f_{\theta'}}(z) dz.
\end{equation*}
Because of the strong convexity of $\ell(f_{\theta}(x),y)$ in $f_{\theta}(x)$ and the fact that $f_{G(\theta)}$ minimizes $h$, we can show that

\begin{align} \label{ps:eq1}
    & -\gamma \|f_{G(\theta)} - f_{G(\theta')}\|_{f_{\theta}}^2 \ge \notag \\ & \int \left(f_{G(\theta)}(x) - f_{G(\theta')}(x)\right)^\top \nabla_{\hat{y}}\ell(f_{G(\theta')}(x),y) p_{f_{\theta}}(z)dz,
\end{align}
where 
\[\|f_{G(\theta)} - f_{G(\theta')}\|_{f_{\theta}}^2 = \int \| f_{G(\theta)}(x) - f_{G(\theta')}(x)\|^2 p_{f_{\theta}}(x) dx.\]
Because of this $f_{\theta}$-dependent norm, assumption~\eqref{A2} is required so we can remove this dependency later.

Using $\epsilon$-sensitivity of $\mathcal{D}(.)$ w.r.t the $\chi^2$ divergence, and the bounded gradient norm assumption which states that there exists a finite value $M$ such that \mbox{$M = \sup_{x,y,\theta} \|\nabla_{\hat{y}}\ell(f_{\theta}(x),y)\|$}, alongside the fact that $f_{G(\theta')}$ minimizes $h'$, we derive that
\begin{align} \label{ps:eq2}
   & \int \left(f_{G(\theta)}(x) - f_{G(\theta')}(x)\right)^\top \nabla_{\hat{y}}\ell(f_{G(\theta')}(x),y) p_{f_{\theta}}(z) dz \ge \notag \\ & - M\sqrt{\epsilon}  \|f_{G(\theta)} - f_{G(\theta')}\|_{f_{\theta}} \|f_{\theta}-f_{\theta'}\|,
\end{align}
which provides a lower bound on the RHS of (\ref{ps:eq1}).

Combining (\ref{ps:eq1}) and (\ref{ps:eq2}) with the fact that the distribution map $\mathcal{D}(.)$ satisfies the bounded norm ratio property with parameter $C$ results in
\begin{equation*}
    \|f_{G(\theta)} - f_{G(\theta')}\|  \le \frac{\sqrt{C\epsilon}M}{\gamma} \|f_{\theta}-f_{\theta'}\|.
\end{equation*}
So for $\frac{\sqrt{C\epsilon}M}{\gamma}<1$, RRM converges to a stable classifier based on the Banach fixed-point theorem.

Setting $\theta=\theta_{t-1}$ and $\theta'=\theta_{\text{PS}}$, we can show that this convergence has a linear rate.
\end{proofsketch}

\begin{figure*}[t!]
	\centering
 		\includegraphics[width=0.9\textwidth]{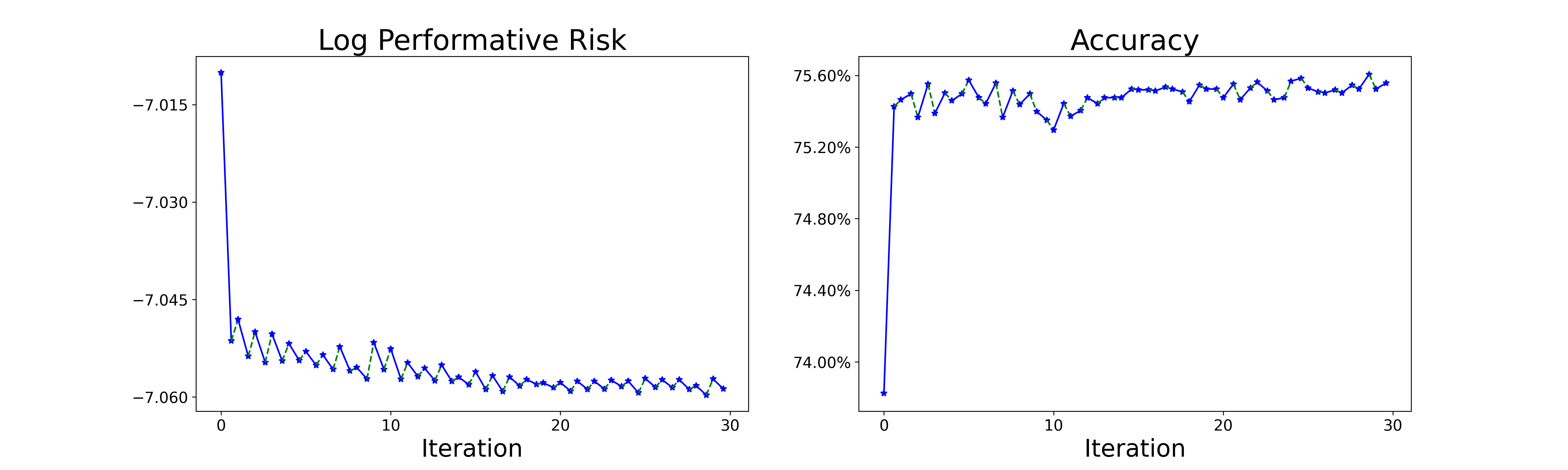}
	\caption{Evolution of log of performative risk (left) and accuracy (right) through iterations of RRM for $\delta=0.9.$  The blue lines show the changes in risk (accuracy) after optimizing on the distribution induced by the last model, and the green lines show the effect of the distribution shift on the risk (accuracy).}
 \label{fig:RRM_convergence}
\end{figure*} 

\section{$\epsilon$-SENSITIVITY OF THE RIR PROCEDURE} \label{RIR}
An example of strategic classification, which was introduced in Section~\ref{contribution}, occurs in social media when users’ posts get rejected because they violated the platform’s policies. In these cases, users usually re-post the same content but with different words in order to get accepted. Inspired by this application, we propose the \textit{Resample-if-Rejected (RIR)} procedure to model distribution shifts. Consider that we have a base distribution with pdf $p$ and a function $g: f_{\theta}(x) \mapsto g(f_{\theta}(x))$ that indicates the probability of rejection. Here we assume that $f_{\theta}(x)$ is a scalar. Let $\text{RIR}(f_{\theta})$ be the distribution resulting from deploying a model with prediction function $f_{\theta}$ under this procedure, and take $p_{f_{\theta}}$ as its pdf. The sampling procedure of $p_{f_{\theta}}$ is as follows:
\begin{itemize}
    \item Take a sample $x^*$ from $p$.
    \item Toss a coin whose probability of getting a head is \mbox{$1-g(f_{\theta}(x^*))$}. If it comes head, output $x^*$, and if it comes tail, output another sample from $p$.
\end{itemize}
Consider $X$ to be a random variable with probability distribution $p(x)$. $p_{f_{\theta}}$ is defined mathematically as
\begin{align*}
        p_{f_{\theta}}(x) & = p(x) \Big(1-g(f_{\theta}(x))\Big) + p(x) \mathbb{E}_X[g(f_{\theta}(X))]
      \\ & = p(x) (1-g(f_{\theta}(x)) + C_{\theta}),
\end{align*}
where $C_{\theta} = \mathbb{E}_X[g(f_{\theta}(X))] = \int g(f_{\theta}(x')) p(x') dx'$.

The following theorem shows that the distribution resulting from the RIR procedure satisfies our conditions on the distribution map. This Theorem is proved in the Supplementary Materials.

\begin{theorem}\label{thm3}
If $f_{\theta}(x)\in [0,1-\delta],\  \forall \theta\in\Theta$ for some fixed \mbox{$0<\delta<1$}, then for $g(f_{\theta}(x)) = f_{\theta}(x) + \delta$, $\text{RIR}(.)$ is \mbox{$\frac{1}{\delta}$-sensitive} w.r.t $\chi^2$ divergence (\aref{A1}) and satisfies the bounded norm ratio property for $C=\frac{1}{\delta}$ \eqref{A2} and \mbox{$c=\frac{1}{2-\delta}$ \eqref{A2prime}.}
\end{theorem}

\begin{remark} \label{remark:RIR}
    Consider a strategic classification task where the distribution reacts to a model with the prediction function $f_{\theta}$ in accordance with the RIR procedure. Suppose the predictions satisfy $f_{\theta}(x)\in[0,1-\delta]$, the label $y$ is in $\{0,1-\delta\}$, and we use the Squared Error loss \mbox{$\ell(f_{\theta}(x),y) = \frac{1}{2} (f_{\theta}(x) - y)^2$} which is 1-strongly convex. According to Theorem~\ref{thm3}, the distribution map is \mbox{$\frac{1}{\delta}$-sensitive} w.r.t $\chi^2$ divergence  and satisfies the bounded norm ratio property for $C=\frac{1}{\delta}$ and $c = \frac{1}{2-\delta}$. Also, \mbox{$M = \sup_{x,y, \theta} |\ell'(f_{\theta}(x),y)|$} is equal to \mbox{$\sup_{x,y, \theta}|f_{\theta}(x)-y|= 1-\delta$}. Putting all these together, the convergence rate of RRM in Theorem~\ref{thm} is equal to $\frac{\sqrt{C\epsilon}M}{\gamma} = \frac{1-\delta}{\delta}$ for part (a) and $\frac{2\sqrt{C\epsilon}M}{\gamma} = \frac{2(1-\delta)}{\delta}$ for part (b), which are less than 1 for $\delta>0.5$ and $\delta>0.67$ respectively.
    
\end{remark}

We will use this remark in the experiments section.

In supervised learning, $x$ corresponds to a set of features. So far in the RIR procedure, we resample the whole set of features, though we can also resample only the strategic features, and Theorem~\ref{thm3} still holds in this case if strategic and non-strategic features are independent as shown following the proof of this theorem in the Supplementary materials~\ref{proof_thm_RIR}. In our simulations in the next section, $x\in \mathbb{R}^d$ is the set of features of individuals applying to get loans from a bank. These features are divided into two sets: $\textit{strategic}$ and $\textit{non-strategic}$. Strategic features are those that can be (easily) manipulated without affecting the true label, e.g., the number of open credit lines and loans. Non-strategic features, however, can be seen as causes of the label and include monthly income for example. In our experiments, we resample only strategic features as these are the ones that people can manipulate more easily.

\section{EXPERIMENTS} \label{experiments}
We complement our theoretical results with experiments on a credit-scoring task and illustrate how they support our claims. We implemented our simulations based on \citet{performative}'s code in the Whynot Python package \citep{miller2020whynot} and changed it according to our settings so that we can use the auto-differentiation of PyTorch\footnote{\url{https://github.com/mhrnz/Performative-Prediction-with-Neural-Networks}}. The strategic classification task of credit scoring is a two-player game between a bank that predicts the creditworthiness of loan applicants, and individuals who strategically manipulate their features to alter the classification outcome. We run the simulations using Kaggle's \textit{Give Me Some Credit} dataset \citep{Kaggle}, with features $x\in \mathbb{R}^{11}$ corresponding to applicants' information along with their label $y\in\{0,1\}$, where $y=1$ indicates that the applicant defaulted and $y=0$ otherwise.

In our simulations, we assume that the data distribution induced by the classifier $f_{\theta}$ shifts according to the RIR procedure where strategic features are resampled with the probability of rejection $g(f_{\theta}(x)) = f_{\theta}(x)+\delta$. Assuming that strategic and non-strategic features are independent, resampling strategic features can be implemented by simply choosing these features from another data point at random. For the classifier, we used a two-layer neural network with a hidden-layer size of 6. The choice of hidden layer size in our network was arbitrary; Figure \ref{fig:diff_hidden_size} shows convergence for different hidden size values. In the network, we use a LeakyReLU activation after the first layer, and a scaled-sigmoid activation function after the second layer to bring the outcome $f_{\theta}(x)$ to the interval $[0, 1-\delta]$. This way we make sure that $g(f_{\theta}(x))\in[\delta,1]$ is a valid probability and the assumption of Theorem~\ref{thm3} is satisfied. Since the outcome $f_{\theta}(x)$ is in $[0,1-\delta]$, we change the label $1$ to $1-\delta$.
So $y=1-\delta$ corresponds to default, and the higher the value of $f_{\theta}(x)$, the greater the chance of rejection.
The objective is to minimize the expectation of the Squared Error loss function over instances, i.e. $\mathbb{E}[\frac{1}{2}(f_{\theta}(x) - y)^2]$. 

The definition of RRM requires solving an exact minimization problem at each optimization step; however, we solve this optimization problem approximately using several steps of gradient descent until the absolute difference of two consecutive risks is less than the tolerance of $10^{-9}$. Also, note that running the same configuration twice might result in different plots because of the randomness that exists in the resampling phase.

\begin{figure*}[t!]
\centering \includegraphics[width=0.8\textwidth]{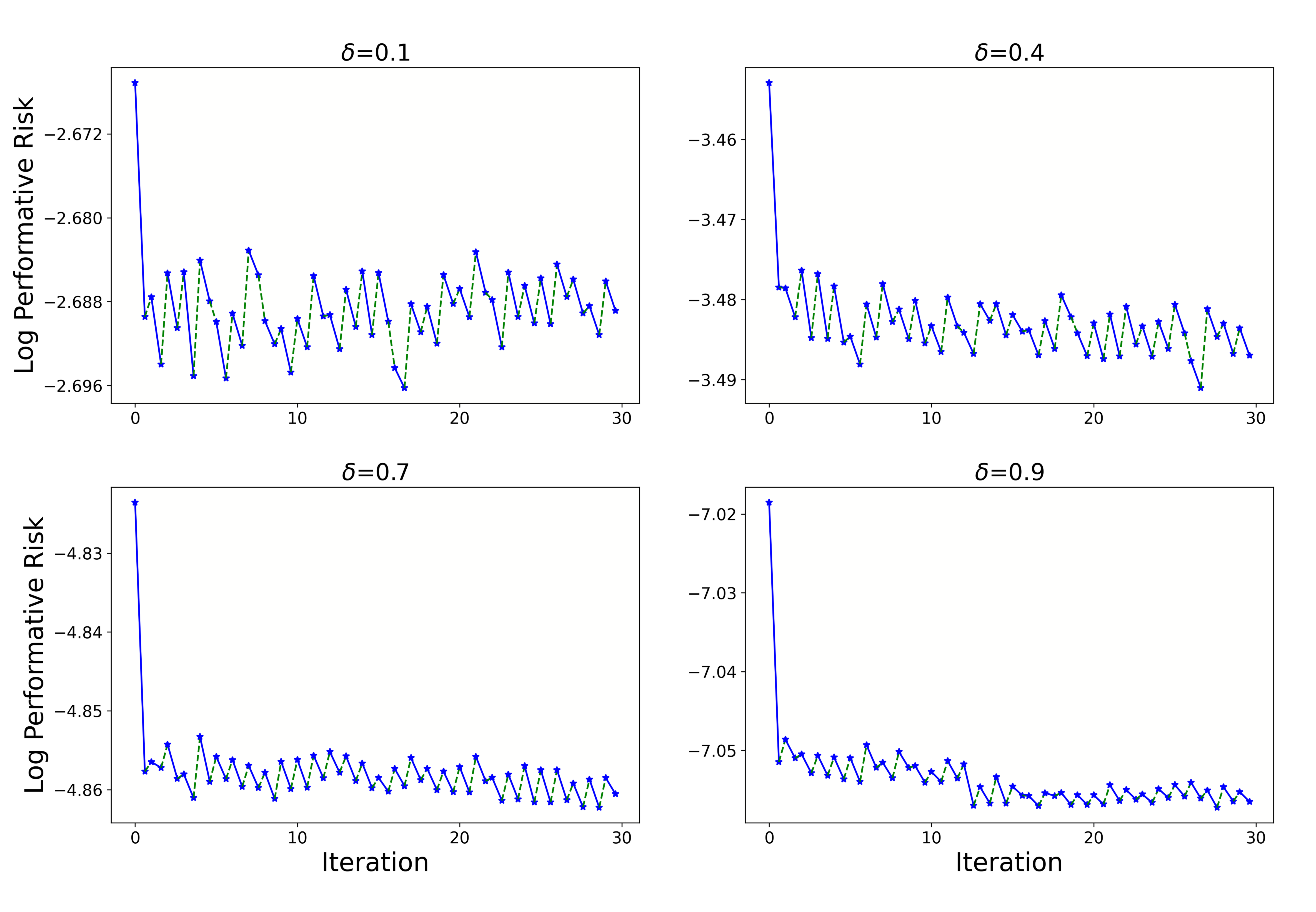}
\caption{Evolution of log of performative risk for different values of $\delta=0.1, 0.4, 0.7, 0.9$ through iterations of RRM.}
\label{fig:diff_delta}
\end{figure*} 

\begin{figure*}[t!]
\centering \includegraphics[width=0.8\textwidth]{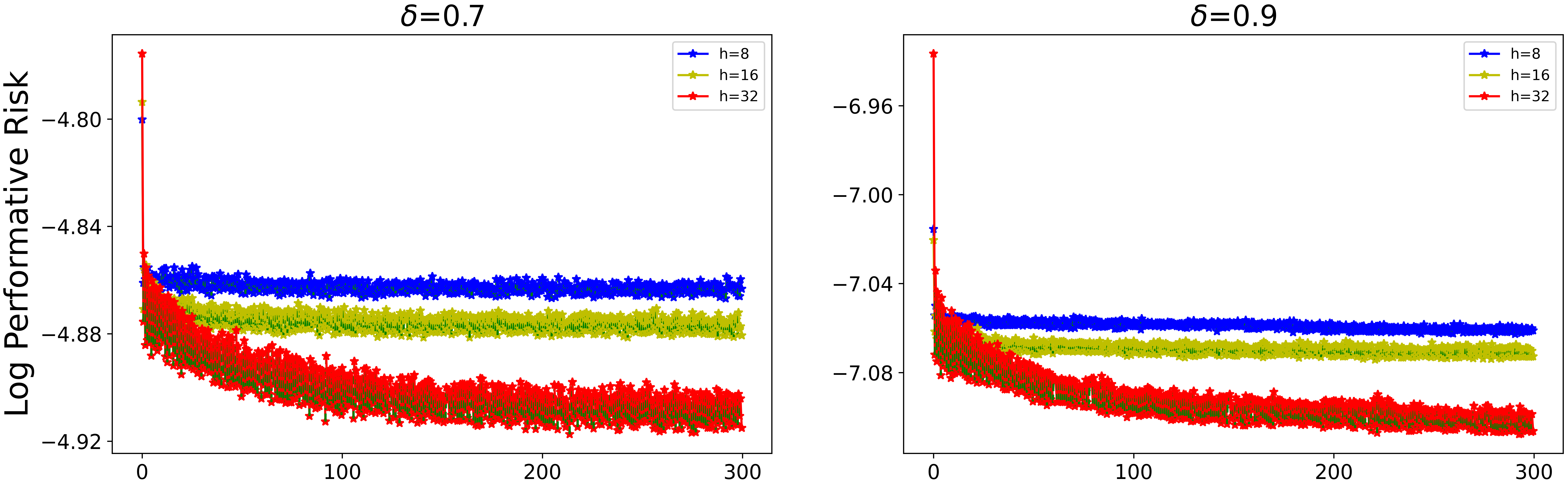}
\caption{Evolution of log of performative risk through iterations of RRM for different values of hidden size $h=8,16,32$ for $\delta=0.7$ and $\delta=0.9$.}
\label{fig:diff_hidden_size}
\end{figure*} 

Figure \ref{fig:RRM_convergence} shows the evolution of the log of performative risk (left) and accuracy (right) through iterations of RRM for \mbox{$\delta=0.9$}.  The blue lines show the changes in risk (accuracy) after optimizing on the distribution induced by the last model, and the green lines show the effect of the distribution shift on the risk (accuracy). We chose to plot the log of performative risk instead of its own value only for illustration purposes. Figure \ref{fig:diff_delta} shows the log of Performative Risk for different values of $\delta=0.1, 0.4, 0.7, 0.9$. The plot for $\delta=0.9$ is generated through a different run than Figure \ref{fig:RRM_convergence}. Based on Remark~\ref{remark:RIR}, for $\delta=0.7$ and $\delta=0.9$ we should see convergence behavior, though for $\delta<0.67$, our theory neither gives a guarantee of convergence nor claims that repeated retraining will diverge, so we might or might not see convergence behavior for $\delta=0.1$ or $\delta=0.4$. What we see in Figure \ref{fig:diff_delta} is aligned with our expectations. It is important to note that for smaller $\delta$, the value of $\epsilon$ which indicates the strength of performative effects is larger, and for high performative effects, it is more difficult for the model to converge since the distribution is allowed to move more after the model's deployment.

On a high level, we interpret the stable classifier to be a model that relies less on non-strategic features for classification. Throughout the training, for a fixed data point $z=(x,y)$ where $x=(x_s, x_f)$ for $x_s$ being the strategic features and $x_f$ being the non-strategic ones, the model sees the same $x_f$ but different values for $x_s$ chosen randomly, all with the same label $y$. So intuitively, the model would learn to rely less on strategic features and more on non-strategic ones for classification, and this makes it more robust to the strategic behavior of agents.

\section{DISCUSSION AND FUTURE WORK}

In this paper, we contribute the first set of convergence guarantees for finding performative stable models on problems where the risk is allowed to be non-convex with respect to parameters. 
This is an important development: our results pertain to modern machine learning models, like neural networks.

We achieve these stronger results by appealing to functional analytical tools, but also making slightly stronger assumptions on the performative feedback loop:
rather than assuming that the distribution is $\eps$-sensitive to parameters as measured by Wasserstein distance, 
we instead assume that the distribution is $\eps$-sensitive to {\em predictions} as measured by the $\chi^2$ divergence.

On one hand, only assuming sensitivity to predictions instead of parameters is a step in the right direction. 
None of the big applications of performative prediction justify sensitivity to model parameters.
As a matter of fact, many of the applications would motivate moving one step further in that direction: 
performative behavior in machine learning systems often manifests as a function of {\em decisions} or {\em actions} that rely on a prediction.
Those decisions or actions are observed by a population that reacts by changing its behavior.
We leave this important problem setting of studying sensitivity to decisions for future work.

On the other hand, $\chi^2$ sensitivity is stronger and implies Wassertstein sensitivity. 
Furthermore, because we use a variable norm that depends on parameters in our analysis, we make an extra assumption that the fixed norm is bounded by the coefficients of this variable norm.
While we provide in Section~\ref{RIR} a well-motivated, concrete example of a performative problem that satisfies both of these conditions, it is nonetheless an interesting open question to wonder how much our analytical assumptions can be loosened.

\section{SOCIETAL IMPACT}
We believe that the deployment of models that can have an impact on the behavior of people (i.e., are performative) should be considered with care, more especially for some critical applications such as elections or regulation of content on social media platforms. 
Our work proposes a new analysis of an existing algorithm that aims at learning performatively stable classifiers. Since the nature of our work is mainly theoretical and does not introduce new algorithms, it does not have a direct societal impact beyond the one described in the original paper on performative prediction. However, since our work supports the use of much more powerful models (e.g., NNs) in performative problems, and this increased power comes with increased responsibility, we should be mindful of the potential for undue influence on society while using this framework.

\subsubsection*{Acknowledgements}
We thank Simon Lacoste-Julien, Quentin Bertrand, Michael Przystupa, and Reza Bayat for their helpful feedback on the manuscript. We would also like to thank Amjad Almahairi for our insightful discussions. 
Ioannis Mitliagkas acknowledges support by an NSERC Discovery grant (RGPIN-2019-06512), a Samsung grant, \mbox{and a Canada CIFAR AI chair}.
\bibliography{paper.bib}
\bibliographystyle{abbrvnat}

\clearpage
\appendix
\onecolumn
{
\aistatstitle{Supplementary Materials: Performative Prediction with Neural Networks}
\section{PROOF OF THEOREMS}
\subsection{Proof of Theorem \ref{thm}}\label{proof_main_thm}
For proving Theorem~\ref{thm}, we were inspired by the proof of Theorem 3.5 in \cite{performative}, but our proof is significantly different from theirs since our analysis is dependent on the prediction function and we need to use infinite-dimensional optimization.
\begin{proof}
Fix \(\theta\) and \(\theta'\) in \(\Theta\). Let \(h:\mathcal{F} \to \mathbb{R}\) and \(h':\mathcal{F} \to \mathbb{R}\) be two functionals defined as follows: 
\begin{equation} \label{eq:a1}
h(f_{\hat{\theta}}) = E_{z\sim \mathcal{D}(f_{\theta})} [\ell(f_{\hat{\theta}}(x),y)] = \int \ell(f_{\hat{\theta}}(x),y) p_{f_{\theta}}(z) dz
\end{equation}
\begin{equation} \label{eq:a2}
h'(f_{\hat{\theta}}) = E_{z\sim \mathcal{D}(f_{\theta'})} [\ell(f_{\hat{\theta}}(x),y)] = \int \ell(f_{\hat{\theta}}(x),y) p_{f_{\theta'}}(z) dz
\end{equation}
where each data point $z$ is a pair of features $x$ and label $y$.

For a fixed \(z=(x,y)\), due to strong convexity of \(\ell(f_{\theta}(x),y)\) in $f_{\theta}(x)$ we have:
\begin{equation} \label{eq:a3}
\ell(f_{G(\theta)}(x),y) - \ell(f_{G(\theta')}(x),y) \ge \left(f_{G(\theta)}(x) - f_{G(\theta')}(x)\right)^\top \nabla_{\hat{y}}\ell(f_{G(\theta')}(x),y) + \frac{\gamma}{2} \| f_{G(\theta)}(x) - f_{G(\theta')}(x)\|^2.
\end{equation}
Now take integral over \(z\), and define \(\|f_{G(\theta)} - f_{G(\theta')}\|_{f_{\theta}}^2 = \int \| f_{G(\theta)}(x) - f_{G(\theta')}(x)\|^2 p_{f_{\theta}}(z) dz\) \big(which is equal to 
\(\int \| f_{G(\theta)}(x) - f_{G(\theta')}(x)\|^2 p_{f_{\theta}}(x) dx\)\big):
\begin{equation} \label{eq:a4}
h(f_{G(\theta)}) - h(f_{G(\theta')}) \ge \left(\int \left(f_{G(\theta)}(x) - f_{G(\theta')}(x)\right)^\top \nabla_{\hat{y}}\ell(f_{G(\theta')}(x),y) p_{f_{\theta}}(z) dz \right)+ \frac{\gamma}{2} \|f_{G(\theta)} - f_{G(\theta')}\|_{f_{\theta}}^2.
\end{equation}
Similarly:
\begin{equation} \label{eq:a5}
h(f_{G(\theta')}) - h(f_{G(\theta)}) \ge \left(\int \left(f_{G(\theta')}(x) - f_{G(\theta)}(x)\right)^\top \nabla_{\hat{y}}\ell(f_{G(\theta)}(x),y) p_{f_{\theta}}(z) dz \right)+ \frac{\gamma}{2} \|f_{G(\theta)} - f_{G(\theta')}\|_{f_{\theta}}^2.
\end{equation}
To ensure clarity of the proof, we complete the proofs of parts (a) and (b) individually in the following sections, despite their significant similarities.
\vfill
\subsubsection{part (a)}
Adding equations (\ref{eq:a4}) and (\ref{eq:a5}) and knowing that \(f_{G(\theta)}\) minimizes \(h\), it is enough to show that
\begin{equation} \label{eq:a6}
\int \left(f_{G(\theta')}(x) - f_{G(\theta)}(x)\right)^\top \nabla_{\hat{y}}\ell(f_{G(\theta)}(x),y) p_{f_{\theta}}(z) dz \ge 0.
\end{equation}
to conclude:
\begin{equation} \label{eq:a7}
-\gamma \|f_{G(\theta)} - f_{G(\theta')}\|_{f_{\theta}}^2 \ge \int \left(f_{G(\theta)}(x) - f_{G(\theta')}(x)\right)^\top \nabla_{\hat{y}}\ell(f_{G(\theta')}(x),y) p_{f_{\theta}}(z) dz.
\end{equation}
This is a key inequality that we will use later in the proof.

Now let's prove inequality (\ref{eq:a6}). Let \(\eta = f_{G(\theta')} - f_{G(\theta)}\). Since the function space is convex, for every \(\alpha\in [0,1]\), \(f_{G(\theta)} + \alpha \eta\) is in the function space. We know that \(f_{G(\theta)}\) is a minimizer of \(h\), so for any $\delta\geq 0$ we have 
\begin{equation}
     h(f_{G(\theta)} + \alpha \eta)  \geq h(f_{G(\theta)}).
\end{equation}
Thus, under the assumption that  $\hat y\mapsto \ell(\hat y,y)\, \forall y$ is continuously differentiable and that there exists $M<\infty$ such that $M = \sup_{x,y,\theta} \|\nabla_{\hat{y}}\ell(f_{\theta}(x),y)\|$, we have that $L:\alpha \mapsto \ell(f_{G(\theta)}(x) + \alpha \eta(x),y)$ is continuously differentiable with $|L'(\alpha,x,y)| =  |\eta(x)^\top \nabla_{\hat{y}}\ell(f_{G(\theta)}(x)+\alpha \eta(x),y)| \leq \|\eta(x)\| \cdot M$ which is an integrable function independent of $\alpha$ (with respect to the distribution $p_{f_\theta}$)
we can use the standard differentiation under the integral lemma \citep[Theorem 6.28]{klenke2013probability} to get that,
\begin{align} \label{eq:a12}
0 & \leq \lim_{\alpha\rightarrow 0} \frac{h(f_{G(\theta)} + \alpha \eta)  - h(f_{G(\theta)})}{\alpha} \notag \\
& = \lim_{\alpha\rightarrow 0} \int \frac{ \ell(f_{G(\theta)}(x) + \alpha \eta(x),y) - \ell(f_{G(\theta)}(x),y)}{\alpha}p_{f_{\theta}}(z)dz \notag \\
& =   \int \lim_{\alpha\rightarrow 0} \frac{ \ell(f_{G(\theta)}(x) + \alpha \eta(x),y) - \ell(f_{G(\theta)}(x),y)}{\alpha}p_{f_{\theta}}(z)dz \notag \\
& = \int \eta(x)^\top \nabla_{\hat{y}}\ell(f_{G(\theta)}(x),y)p_{f_{\theta}}(z)dz \notag \\
& = \int (f_{G(\theta')}(x) - f_{G(\theta)}(x))^\top \nabla_{\hat{y}}\ell(f_{G(\theta)}(x),y)p_{f_{\theta}}(z)dz.
\end{align}
This proves inequality (\ref{eq:a6}).
Now recall that the distribution map over data is \(\epsilon\)-sensitive w.r.t Pearson $\chi^2$ divergence, i.e.
\begin{equation} \label{eq:a13}
\chi^2(\mathcal{D}(f_{\theta'}), \mathcal{D}(f_{\theta})) \le \epsilon \|f_{\theta}-f_{\theta'}\|^2.
\end{equation}
With this in mind, we do the following calculations:
\begin{align*}
&\left| \int \left(f_{G(\theta)}(x) - f_{G(\theta')}(x)\right)^\top \nabla_{\hat{y}}\ell(f_{G(\theta')}(x),y) p_{f_{\theta}}(z) dz -  \int \left(f_{G(\theta)}(x) - f_{G(\theta')}(x)\right)^\top\nabla_{\hat{y}}\ell(f_{G(\theta')}(x),y) p_{f_{\theta'}}(z) dz\right|\\
& = \left| \int \left(f_{G(\theta)}(x) - f_{G(\theta')}(x)\right)^\top\nabla_{\hat{y}}\ell(f_{G(\theta')}(x),y) \left(p_{f_{\theta}}(z)-p_{f_{\theta'}}(z)\right) dz \right|\\
& \overset{(*)}{\le} \int  \left| \left(f_{G(\theta)}(x) - f_{G(\theta')}(x)\right)^\top \nabla_{\hat{y}}\ell(f_{G(\theta')}(x),y) \left(p_{f_{\theta}}(z)-p_{f_{\theta'}}(z)\right)\right| dz \\
& \le M \int  \left| \|f_{G(\theta)}(x) - f_{G(\theta')}(x)\| \left(p_{f_{\theta}}(z)-p_{f_{\theta'}}(z)\right)\right| dz \\
& = M \int\left| \|f_{G(\theta)}(x) - f_{G(\theta')}(x)\| \frac{p_{f_{\theta}}(z)-p_{f_{\theta'}}(z)}{p_{f_{\theta}}(z)} p_{f_{\theta}}(z) \right| dz\\
& = M \left| \int \big| \|f_{G(\theta)}(x) - f_{G(\theta')}(x)\| \frac{p_{f_{\theta}}(z)-p_{f_{\theta'}}(z)}{p_{f_{\theta}}(z)} \big| p_{f_{\theta}}(z)dz \right|\\
& \overset{\text{Cauchy-Schwarz Ineq.}}{\le} M \left(\int \|f_{G(\theta)}(x) - f_{G(\theta')}(x)\|^2 p_{f_{\theta}}(z) dz\right)^{\frac{1}{2}} \left(\int \left(\frac{p_{f_{\theta}}(z)-p_{f_{\theta'}}(z)}{p_{f_{\theta}}(z)}\right)^2 p_{f_{\theta}}(z) dz\right)^{\frac{1}{2}} \\
& = M \|f_{G(\theta)} - f_{G(\theta')}\|_{f_{\theta}} \sqrt{\chi^2(\mathcal{D}(f_{\theta'}), \mathcal{D}(f_{\theta}))} \\
& \overset{(\ref{eq:a13})}{\le} M\sqrt{\epsilon}  \|f_{G(\theta)} - f_{G(\theta')}\|_{f_{\theta}} \|f_{\theta}-f_{\theta'}\|
\end{align*}
\((*)\) comes from the fact that $\left|\int f(x)dx\right|\le \int |f(x)|dx$, and the Cauchy-Schwarz inequality states that $|\mathbb{E} [XY]|\le \sqrt{\mathbb{E}[X^2]\mathbb{E}[Y^2]}$.\\
We conclude from the above derivations that:
\begin{align} \label{eq:a14}
& \left| \int \left(f_{G(\theta)}(x) - f_{G(\theta')}(x)\right)^\top \nabla_{\hat{y}}\ell(f_{G(\theta')}(x),y) p_{f_{\theta}}(z) dz -  \int \left(f_{G(\theta)}(x) - f_{G(\theta')}(x)\right)^\top\nabla_{\hat{y}}\ell(f_{G(\theta')}(x),y) p_{f_{\theta'}}(z) dz\right| \notag \\ & \le M\sqrt{\epsilon}  \|f_{G(\theta)} - f_{G(\theta')}\|_{f_{\theta}} \|f_{\theta}-f_{\theta'}\|.
\end{align}
Similar to inequality (\ref{eq:a6}), since \(f_{G(\theta')}\) minimizes \(h'\), one can prove:
\begin{equation} \label{eq:a15}
\int \left(f_{G(\theta)}(x) - f_{G(\theta')}(x)\right)^\top \nabla_{\hat{y}}\ell(f_{G(\theta')}(x),y) p_{f_{\theta'}}(z) dz \ge 0.
\end{equation}
From (\ref{eq:a7}) we know that \(\int \left(f_{G(\theta)}(x) - f_{G(\theta')}(x)\right)^\top \nabla_{\hat{y}}\ell(f_{G(\theta')}(x),y) p_{f_{\theta}}(z) dz\) is negative, so with this fact alongside (\ref{eq:a14}) and (\ref{eq:a15}), we can write:
\begin{equation} \label{eq:a16}
\int \left(f_{G(\theta)}(x) - f_{G(\theta')}(x)\right)^\top \nabla_{\hat{y}}\ell(f_{G(\theta')}(x),y) p_{f_{\theta}}(z) dz \ge - M\sqrt{\epsilon}  \|f_{G(\theta)} - f_{G(\theta')}\|_{f_{\theta}} \|f_{\theta}-f_{\theta'}\|.
\end{equation}
Combining (\ref{eq:a7}) and (\ref{eq:a16}), we will get:
\begin{align} \label{eq:a17}
&\gamma \|f_{G(\theta)} - f_{G(\theta')}\|_{f_{\theta}}^2 \le M\sqrt{\epsilon}  \|f_{G(\theta)} - f_{G(\theta')}\|_{f_{\theta}} \|f_{\theta}-f_{\theta'}\| \notag \\
& \Rightarrow \|f_{G(\theta)} - f_{G(\theta')}\|_{f_{\theta}} \le \frac{\sqrt{\epsilon}M}{\gamma} \|f_{\theta}-f_{\theta'}\|
\end{align}
Since the distribution map satisfies the bounded norm ratio assumption with parameter $C$, we can write:
\begin{align}
\|f_{G(\theta)} - f_{G(\theta')}\|^2 & = \int \|f_{G(\theta)}(x) - f_{G(\theta')}(x)\|^2 p(x) dx \notag \\& \le C \int \|f_{G(\theta)}(x) - f_{G(\theta')}(x)\|^2 p_{f_{\theta}}(x) dx \notag \\& = C \|f_{G(\theta)} - f_{G(\theta')}\|_{f_{\theta}}^2
\end{align}
Consequently,
\begin{equation} \label{eq:a19}
\|f_{G(\theta)} - f_{G(\theta')}\| \le \sqrt{C} \|f_{G(\theta)} - f_{G(\theta')}\|_{f_{\theta}}
\end{equation}
Using (\ref{eq:a19}) in (\ref{eq:a17}) results in:
\begin{equation}
    \|f_{G(\theta)} - f_{G(\theta')}\|  \le \frac{\sqrt{C\epsilon}M}{\gamma} \|f_{\theta}-f_{\theta'}\|.
\end{equation}
So if $\frac{\sqrt{C\epsilon}M}{\gamma}<1$, $G$ is a contractive mapping and RRM converges to a stable classifier based on Banach fixed point theorem.

If we set $\theta=\theta_{t-1}$ and $\theta' = \theta_{\text{PS}}$ for $\theta_{\text{PS}}$ being a stable classifier, we know that $G(\theta) = \theta_t$ and $G(\theta') = \theta_{\text{PS}}$. So we will have:
\begin{align} \label{eq:a21}
         \|f_{\theta_t} - f_{\theta_{\text{PS}}}\| & \le \frac{\sqrt{C\epsilon}M}{\gamma} \|f_{\theta_{t-1}}-f_{\theta_{\text{PS}}}\| \notag \\ &
         \le (\frac{\sqrt{C\epsilon}M}{\gamma})^t \|f_{\theta_0}-f_{\theta_{\text{PS}}}\|
\end{align}
We can easily see that for $t\ge (1-\frac{\sqrt{C\epsilon}M}{\gamma})^{-1} \log (\frac{\|f_{\theta_0} - f_{\theta_{\text{PS}}}\|}{\alpha})$, 
\begin{equation*}
    (\frac{\sqrt{C\epsilon}M}{\gamma})^t \|f_{\theta_0}-f_{\theta_{\text{PS}}}\| \le \alpha
\end{equation*}
So based on (\ref{eq:a21}), 
\begin{equation*}
\|f_{\theta_t} - f_{\theta_{\text{PS}}}\| \le \alpha.
\end{equation*}
which shows that RRM converges to a stable classifier at a linear rate.

\subsubsection{part (b)}\label{sec:proof_partb}
Since $f_{G(\theta)}$ minimizes \(h\), $h(f_{G(\theta)}) - h(f_{G(\theta')}) \le 0$, so based on inequality (\ref{eq:a4}) we have:
\begin{equation} \label{eq:b1}
-\frac{\gamma}{2} \|f_{G(\theta)} - f_{G(\theta')}\|_{f_{\theta}}^2 \ge \int \left(f_{G(\theta)}(x) - f_{G(\theta')}(x)\right)^\top \nabla_{\hat{y}}\ell(f_{G(\theta')}(x),y) p_{f_{\theta}}(z) dz.
\end{equation}
We will use this inequality later in the proof.

Now recall that there exists \(M\) such that \(M = \sup_{x,y,\theta} \|\nabla_{\hat{y}}\ell(f_{\theta}(x),y)\|\) and the distribution map over data is \(\epsilon\)-sensitive w.r.t Pearson $\chi^2$ divergence, i.e.
\begin{equation} \label{eq:b13}
\chi^2(\mathcal{D}(f_{\theta'}), \mathcal{D}(f_{\theta})) \le \epsilon \|f_{\theta}-f_{\theta'}\|^2.
\end{equation}
With this in mind, we do the following calculations:
\begin{align*}
&\left| \int \left(f_{G(\theta)}(x) - f_{G(\theta')}(x)\right)^\top \nabla_{\hat{y}}\ell(f_{G(\theta')}(x),y) p_{f_{\theta}}(z) dz -  \int \left(f_{G(\theta)}(x) - f_{G(\theta')}(x)\right)^\top\nabla_{\hat{y}}\ell(f_{G(\theta')}(x),y) p_{f_{\theta'}}(z) dz\right|\\
& = \left| \int \left(f_{G(\theta)}(x) - f_{G(\theta')}(x)\right)^\top\nabla_{\hat{y}}\ell(f_{G(\theta')}(x),y) \left(p_{f_{\theta}}(z)-p_{f_{\theta'}}(z)\right) dz \right|\\
& \overset{(*)}{\le} \int  \left| \left(f_{G(\theta)}(x) - f_{G(\theta')}(x)\right)^\top \nabla_{\hat{y}}\ell(f_{G(\theta')}(x),y) \left(p_{f_{\theta}}(z)-p_{f_{\theta'}}(z)\right)\right| dz \\
& \le M \int  \left| \|f_{G(\theta)}(x) - f_{G(\theta')}(x)\| \left(p_{f_{\theta}}(z)-p_{f_{\theta'}}(z)\right)\right| dz \\
& = M \int\left| \|f_{G(\theta)}(x) - f_{G(\theta')}(x)\| \frac{p_{f_{\theta}}(z)-p_{f_{\theta'}}(z)}{p_{f_{\theta}}(z)} p_{f_{\theta}}(z) \right| dz\\
& = M \left| \int \big| \|f_{G(\theta)}(x) - f_{G(\theta')}(x)\| \frac{p_{f_{\theta}}(z)-p_{f_{\theta'}}(z)}{p_{f_{\theta}}(z)} \big| p_{f_{\theta}}(z)dz \right|\\
& \overset{\text{Cauchy-Schwarz Ineq.}}{\le} M \left(\int \|f_{G(\theta)}(x) - f_{G(\theta')}(x)\|^2 p_{f_{\theta}}(z) dz\right)^{\frac{1}{2}} \left(\int \left(\frac{p_{f_{\theta}}(z)-p_{f_{\theta'}}(z)}{p_{f_{\theta}}(z)}\right)^2 p_{f_{\theta}}(z) dz\right)^{\frac{1}{2}} \\
& = M \|f_{G(\theta)} - f_{G(\theta')}\|_{f_{\theta}} \sqrt{\chi^2(\mathcal{D}(f_{\theta'}), \mathcal{D}(f_{\theta}))} \\
& \overset{(\ref{eq:b13})}{\le} M\sqrt{\epsilon}  \|f_{G(\theta)} - f_{G(\theta')}\|_{f_{\theta}} \|f_{\theta}-f_{\theta'}\|
\end{align*}
\((*)\) comes from the fact that $\left|\int f(x)dx\right|\le \int |f(x)|dx$, and the Cauchy-Schwarz inequality states that \mbox{$|\mathbb{E} [XY]|\le \sqrt{\mathbb{E}[X^2]\mathbb{E}[Y^2]}$}.\\
We conclude from the above derivations that:
\begin{align} 
& \left| \int \left(f_{G(\theta)}(x) - f_{G(\theta')}(x)\right)^\top \nabla_{\hat{y}}\ell(f_{G(\theta')}(x),y) p_{f_{\theta}}(z) dz -  \int \left(f_{G(\theta)}(x) - f_{G(\theta')}(x)\right)^\top\nabla_{\hat{y}}\ell(f_{G(\theta')}(x),y) p_{f_{\theta'}}(z) dz\right| \notag \\ & \le M\sqrt{\epsilon}  \|f_{G(\theta)} - f_{G(\theta')}\|_{f_{\theta}} \|f_{\theta}-f_{\theta'}\| \notag.
\end{align}
Therefore:
\begin{align}\label{eq:b14}
    & \int \left(f_{G(\theta)}(x) - f_{G(\theta')}(x)\right)^\top \nabla_{\hat{y}}\ell(f_{G(\theta')}(x),y) p_{f_{\theta}}(z) dz \ge 
    \int \left(f_{G(\theta)}(x) - f_{G(\theta')}(x)\right)^\top\nabla_{\hat{y}}\ell(f_{G(\theta')}(x),y) p_{f_{\theta'}}(z) dz \notag\\
    & - M\sqrt{\epsilon}  \|f_{G(\theta)} - f_{G(\theta')}\|_{f_{\theta}} \|f_{\theta}-f_{\theta'}\|.
\end{align}

We lower bound the integral on the RHS in Proposition~\ref{prop:delta_approx}, but before stating the proposition, we need two lemmas, which are stated and proved below.
\begin{lemma} \label{lemma:f*def}
    Let $f^*_{\theta'} = \arg \min\limits_{f\in \mathcal{M}} \mathbb{E}_{z\sim p_{f_{\theta'}}} [\ell(f(x),y)]$, where $\mathcal{M}$ is the set of all measurable functions from $\mathcal{X}$ to $\mathcal{Y}$. For the squared-error loss function, $f^*_{\theta'}(x) = \int_y y p_{f_{\theta'}}(y|x)dy = \mathbb{E}_{p_{f_{\theta'}}}[Y|X=x]$.
\end{lemma}
\begin{proof}[proof of Lemma \ref{lemma:f*def}]
We know that $f^*_{\theta'} = \argmin\limits_{f\in \mathcal{M}} \int \ell(f(x),y) p_{f_{\theta'}}(z) dz.$
\begin{equation}\label{eq:lemma1_1}
    f^*_{\theta'} = \argmin\limits_{f\in \mathcal{M}} \int_x \int_y \ell(f(x),y) p_{f_{\theta'}}(y|x) dy p_{f_{\theta'}}(x) dx
\end{equation}
To minimize this, it's enough to find $f^*_{\theta'}$ such that $f^*_{\theta'}(x) = \argmin_{\hat{y}} \int_y \ell(\hat{y},y) p_{f_{\theta'}}(y|x) dy$. For $\ell(\hat{y},y) = \frac{1}{2} \|\hat{y}-y\|^2$, this is minimized at $\hat{y} = \int_y y p_{f_{\theta'}}(y|x)dy$ by taking the gradient and setting it to zero. 
\end{proof}
\begin{lemma}\label{lemma:argmin}
    $f_{G(\theta')} = \argmin\limits_{f\in \mathcal{F}} \|f - f^*_{\theta'}\|_{f_{\theta'}}$.
\end{lemma} 
\begin{proof}[proof of Lemma \ref{lemma:argmin}]
\begin{align}
    f_{G(\theta')} & = \argmin\limits_{f\in \mathcal{F}} \int_x \int_y \|f(x) - y\|^2 p_{f_{\theta'}}(x) p_{f_{\theta'}}(y|x) dy dx \notag\\
    & = \argmin\limits_{f\in \mathcal{F}} \int_x \int_y \frac{1}{2} [\|f(x)\|^2 - 2 f(x)^\top y + \|y\|^2] p_{f_{\theta'}}(y|x) dy p_{f_{\theta'}}(x) dx \notag\\
    & = \argmin\limits_{f\in \mathcal{F}} \int_x \int_y \frac{1}{2} \|f(x)\|^2 p_{f_{\theta'}}(y|x) dy p_{f_{\theta'}}(x) dx - \int_x \int_y f(x)^\top f^*_{\theta'}(x) p_{f_{\theta'}}(y|x) dy p_{f_{\theta'}}(x) dx \notag\\
    & \quad \quad \quad \quad + \int_x \int_y \frac{1}{2} [\|y\|^2+\|f^*_{\theta'}(x)\|^2-\|f^*_{\theta'}(x)\|^2] p_{f_{\theta'}}(y|x) dy p_{f_{\theta'}}(x) dx \notag\\
    & = \argmin\limits_{f\in \mathcal{F}} \int_x \int_y \frac{1}{2} [\|f(x)\|^2 - 2 f(x)^\top f^*_{\theta'}(x) + \|f^*_{\theta'}(x)\|^2] p_{f_{\theta'}}(y|x) dy p_{f_{\theta'}}(x) dx + \Tilde{C} \notag\\
    & = \argmin\limits_{f\in \mathcal{F}} \frac{1}{2} \| f - f^*_{\theta'}\|_{f_{\theta'}}^2 + \Tilde{C},
\end{align}
where $\Tilde{C}$ does not depend of $f$. Note that in the third line of the proof, we used the definition of $f^*_{\theta'}(x)$ as derived in Lemma~\ref{lemma:f*def}. As a result, $f_{G(\theta')} = \argmin\limits_{f\in \mathcal{F}} \|f - f^*_{\theta'}\|_{f_{\theta'}}$ and this completes the proof.
\end{proof}
\begin{proposition}\label{prop:delta_approx} With assumptions of the theorem, $\int \left(f_{G(\theta)}(x) - f_{G(\theta')}(x)\right)^\top\nabla_{\hat{y}}\ell(f_{G(\theta')}(x),y) p_{f_{\theta'}}(z) dz \ge -\frac{\kappa \sqrt{C}}{c} \|f_{G(\theta)} - f_{G(\theta')}\|_{f_{\theta}}$, where $\ell(f(x),y)$ is the squared error loss, i.e. $\ell(f(x),y)=\frac{1}{2}\|f(x)-y\|^2$.
\end{proposition}
\begin{proof}[proof of Proposition]
For $f^*_{\theta'} = \arg \min\limits_{f\in \mathcal{M}} \mathbb{E}_{z\sim p_{f_{\theta'}}} [\ell(f(x),y)]$, we can derive from Lemma~\ref{lemma:f*def} that $\int_y \nabla_{\hat{y}}\ell(f^*_{\theta'}(x),y) p_{f_{\theta'}}(y|x)dy = \boldsymbol{0}$. Considering this, we can write the following:
\begin{align} \label{eq:b15}
   & \Big| \int \left(f_{G(\theta)}(x) - f_{G(\theta')}(x)\right)^\top\nabla_{\hat{y}}\ell(f_{G(\theta')}(x),y) p_{f_{\theta'}}(z) dz \Big| \notag \\
   & = \Big| \int \left(f_{G(\theta)}(x) - f_{G(\theta')}(x)\right)^\top(\nabla_{\hat{y}}\ell(f_{G(\theta')}(x),y)-\nabla_{\hat{y}}\ell(f^*_{\theta'}(x),y)) p_{f_{\theta'}}(z) dz \Big|  \notag \\
   & \le \int \|f_{G(\theta)}(x) - f_{G(\theta')}(x)\|\|\nabla_{\hat{y}}\ell(f_{G(\theta')}(x),y)-\nabla_{\hat{y}}\ell(f^*_{\theta'}(x),y)\| p_{f_{\theta'}}(z) dz \notag \\
   & \overset{1\text{-smoothness of $\ell$}}{\le}\int \|f_{G(\theta)}(x) - f_{G(\theta')}(x)\|\|f_{G(\theta')}(x) - f^*_{\theta'}(x)\| p_{f_{\theta'}}(z) dz \notag \\
   & \le \|f_{G(\theta)} - f_{G(\theta')}\|_{f_{\theta'}} \|f_{G(\theta')} - f^*_{\theta'}\|_{f_{\theta'}} \notag \\
   & \le \frac{\kappa}{\sqrt{c}} \|f_{G(\theta)} - f_{G(\theta')}\|_{f_{\theta'}}, 
\end{align}
where in the last line we used $\|f_{G(\theta')} - f^*_{\theta'}\|_{f_{\theta'}} \le \frac{\kappa}{\sqrt{c}}$.
Based on the assumption made in Theorem \ref{thm}, $\exists f \in \mathcal{F}$ such that $\|f-f^*_{\theta'}\|\le \kappa$, therefore \mbox{$\|f-f^*_{\theta'}\|_{f_{\theta'}} \le \frac{\kappa}{\sqrt{c}}$ \eqref{A2prime}}. Using this and Lemma \ref{lemma:argmin} together, we conclude that $\|f_{G(\theta')} - f^*_{\theta'}\|_{f_{\theta'}} \le \frac{\kappa}{\sqrt{c}}$.\\
Following inequality \ref{eq:b15} and \eqref{A2}, we can complete the proof of the proposition and write:
$\int \left(f_{G(\theta)}(x) - f_{G(\theta')}(x)\right)^\top\nabla_{\hat{y}}\ell(f_{G(\theta')}(x),y) p_{f_{\theta'}}(z) dz \ge -\frac{ \kappa \sqrt{C}}{c} \|f_{G(\theta)} - f_{G(\theta')}\|_{f_{\theta}}$.
\end{proof}
Before proceeding with the rest of the proof, we will examine the assumption stated in the theorem saying that $\mathcal{F}$ can approximate any target function $f^*_{\theta'} = \mathbb{E}{p_{f_{\theta'}}}[y|X=x]$ with an error of at most $\kappa$. As shown in Remark \ref{remark:sobolev}, if $f^*_{\theta'}$ is Lipschitz continuous, this assumption holds when $\mathcal{F}$ includes a sufficiently large class of neural networks.
\begin{remark}\label{remark:sobolev}
    Consider the Sobolev spaces, $\mathcal{W}^{n,\infty}([0,1]^d)$ with $n=1,2,...$ that consist of the functions from $C^{n-1}([0,1]^d)$ such that all their derivatives of order $n-1$ are Lipschitz continuous and let $F_{n,d}$ be the unit ball in this space. This notation is taken from \citet{yarotsky2016error}, Sec 3.2. Based on Theorem 1 in \citet{yarotsky2016error}, for any $d,n$ and $\epsilon \in (0,1)$ there is a ReLU network architecture that can approximate any function from $F_{n,d}$ with error $\epsilon$.
    Let $F_{A}$ be the approximator function class that is capable of expressing any L-Lipschitz function with error $\kappa$ where L is the Lipschitz constant of $f^*_{\theta'}$ as defined above. The initial assumption in part (b) of Theorem \ref{thm} holds when $\mathcal{F}$ contains $F_A$.
 \end{remark}
Continuing with the proof, we apply Proposition \ref{prop:delta_approx} to inequality \ref{eq:b14}, which combined with inequality \ref{eq:b1} results in the following:
\begin{equation}
    \frac{\gamma}{2}\|f_{G(\theta)} - f_{G(\theta')}\|_{f_{\theta}} \le M\sqrt{\epsilon} \|f_{\theta} - f_{\theta'}\| + \frac{\kappa\sqrt{C}}{c}.
\end{equation}
Since the distribution map satisfies the bounded norm ratio assumption with parameter $C$, we can write:
\begin{equation}\label{eq:contraction}
    \|f_{G(\theta)} - f_{G(\theta')}\| \le \frac{2\sqrt{C \epsilon}M}{\gamma} \|f_{\theta} - f_{\theta'}\| + \frac{2\kappa C}{\gamma c}.
\end{equation}

Now we want to prove the convergence rate to the stable classifier $f^* = \argmin\limits_{f\in \mathcal{M}} \mathbb{E}_{z\sim \mathcal{D}(f^*)} [\ell(f(x),y)],$ where $\mathcal{M}$ is the set of all measurable functions. $f^*$ might not be in our function space $\mathcal{F}$. Since $\mathcal{M}$ is convex, we get from part (a) of the theorem that $f^*$ exists and is unique.

Let $f_{\theta^*}$ be defined as the closest function in $\mathcal{F}$ to $f^*$, i.e., $f_{\theta^*}=\argmin\limits_{f\in\mathcal{F}} \|f-f^*\|.$ Based on Lemma~\ref{lemma:argmin} and the assumption made in Theorem \ref{thm} ($\exists f \in \mathcal{F}$ such that $\|f-f^*_{\theta'}\|\le \kappa$), we will get the following:
\begin{equation}\label{eq:36}
    \|f_{G(\theta^*)} - f_{\theta^*}^*\|_{f_{\theta^*}} \le \frac{\kappa}{\sqrt{c}}.
\end{equation}
Using traingular inequality, we can write 
\begin{equation}\label{eq:convergence_trieq}
    \|f_{\theta_{t+1}} - f^*\| \le \|f_{\theta_{t+1}} - f_{G(\theta^*)}\| + \|f_{G(\theta^*)} - f^*\|.
\end{equation}
Let's process each term in the RHS of this inequality separately:
\begin{itemize}
    \item $\|f_{\theta_{t+1}} - f_{G(\theta^*)}\|$: From inequality~\ref{eq:contraction}, we can write
    \begin{align}
    \|f_{\theta_{t+1}} - f_{G(\theta^*)}\| & \le \frac{2\sqrt{C \epsilon}M}{\gamma} \|f_{\theta_t} - f_{\theta^*}\| + \frac{2\kappa C}{\gamma c} \notag \\
    & \le \frac{2\sqrt{C \epsilon}M}{\gamma} \|f_{\theta_t} - f^*\| + \frac{2\sqrt{C \epsilon}M}{\gamma} \|f^* - f_{\theta^*}\| + \frac{2\kappa C}{\gamma c},
    \end{align}
    where in the second line we again used triangular inequality. Since we assume in the theorem that we can approximate $f^*$ with an error of at most $\kappa$ and knowing that $f_{\theta^*}=\argmin\limits_{f\in\mathcal{F}} \|f-f^*\|$, we can write the following:
    \begin{equation}\label{eq:convergenceRHS_part1}
        \|f_{\theta_{t+1}} - f_{G(\theta^*)}\| \le  \frac{2\sqrt{C \epsilon}M}{\gamma} \|f_{\theta_t} - f^*\| + \frac{2\sqrt{C \epsilon}M\kappa}{\gamma} + \frac{2\kappa C}{\gamma c}.
    \end{equation}
    \item $\|f_{G(\theta^*)} - f^*\|$: Using the triangular inequality once again:
    \begin{equation}\label{eq:40}
        \|f_{G(\theta^*)} - f^*\| \le \|f_{G(\theta^*)} - f_{\theta^*}^*\| + \|f_{\theta^*}^* - f^*\|
    \end{equation}
    Given~\ref{eq:36} and Assumption \eqref{A2}, $ \|f_{G(\theta^*)} - f_{\theta^*}^*\| \le \sqrt{\frac{C}{c}}\kappa$.

    $f_{\theta^*}^* = \argmin\limits_{f\in\mathcal{M}}\mathbb{E}_{z\sim \mathcal{D}(f_{\theta^*})}[\ell(f(x),y)]=f_{G_{\mathcal{M}}(\theta^*)}$, where $G_\mathcal{M}$ here is the function $G$ defined in Definition \ref{def3} when the function space is $\mathcal{M}$ (the notation $G_\mathcal{M}$ is only used here to distinguish this function from the function $G$ used throughout the proof). Since $\mathcal{M}$ is a convex space, we can apply the convergence results from part (a), recalling that $f^*$ is the stable solution and write:
    \begin{equation}
        \|f_{\theta^*}^* - f^*\| \le \frac{\sqrt{C\epsilon}M}{\gamma} \|f_{\theta^*} - f^*\| \le \frac{\sqrt{C\epsilon}M\kappa}{\gamma}
    \end{equation}
    Using these in inequality~\ref{eq:40} we can conclude:
    \begin{equation}\label{eq:convergenceRHS_part2}
         \|f_{G(\theta^*)} - f^*\| \le \big(\sqrt{\frac{C}{c}} + \frac{\sqrt{C\epsilon}M}{\gamma}\big)\kappa.
    \end{equation}
\end{itemize}
Combining \ref{eq:convergenceRHS_part1} and \ref{eq:convergenceRHS_part2} in~\ref{eq:convergence_trieq}, we get the following:
\begin{equation}
    \|f_{\theta_{t+1}} - f^*\| \le \frac{2\sqrt{C \epsilon}M}{\gamma} \|f_{\theta_t} - f^*\| + \big(\frac{3\sqrt{C \epsilon}M}{\gamma} + \frac{2C}{\gamma c} + \sqrt{\frac{C}{c}}\big)\kappa.
\end{equation}
In conclusion, if $\frac{2\sqrt{C \epsilon}M}{\gamma}<1$, RRM converges to a neighborhood around the stable solution at a linear rate where the size of this neighborhood depends on $\kappa$ and diminishes as $\kappa$ decreases:
\begin{align*}
      \|f_{\theta_{t}} - f^*\| & \le \left(\frac{2\sqrt{C \epsilon}M}{\gamma}\right)^t \|f_{\theta_0} - f^*\| + O(\kappa).
\end{align*}
\end{proof}
\newpage
\subsection{Proof of Theorem~\ref{thm3}}\label{proof_thm_RIR}
\begin{proof}
    As explained in Section~\ref{RIR} of the paper, the pdf of $p_{f_{\theta}}(x)$ is as follows:
\begin{align} \label{eq:18}
       p_{f_{\theta}}(x) & = p(x) \Big(1-g(f_{\theta}(x))\Big) + p(x) \mathbb{E}_X[g(f_{\theta}(X))]
       \notag \\ & = p(x) (1-g(f_{\theta}(x)) + C_{\theta}).
\end{align}
where $C_{\theta} =  \mathbb{E}_X[g(f_{\theta}(X))] = \int p(x') g(f_{\theta}(x'))dx'$.

For $g(f_{\theta}(x)) = f_{\theta}(x) + \delta$, we have:
\begin{equation}
    p_{f_{\theta}}(x) = p(x) (1-f_{\theta}(x)-\delta + C_{\theta}).
\end{equation}
where $\delta \le C_{\theta}\le 1$ since $0\le f_{\theta}(x)\le 1-\delta$, so $\delta \le g(f_{\theta}(x))\le 1$ for every $x$.

In the RIR procedure, the distribution of the label $y$ given $x$ is not affected by the predictions so for every $z=(x,y)$ we have $p_{f_{\theta}}(z) = p_{f_{\theta}}(x) p(y|x)$ for any $f_{\theta}$. This results in the following equality:
\begin{equation*}
    \chi^2(\mathcal{D}(f_{\theta'}), \mathcal{D}(f_{\theta})) = \int \frac{(p_{f_{\theta'}}(z) - p_{f_{\theta}}(z))^2}{p_{f_{\theta}}(z)} dz = \int \frac{(p_{f_{\theta'}}(x) - p_{f_{\theta}}(x))^2}{p_{f_{\theta}}(x)} dx
\end{equation*}

Now we can prove that this distribution map is $\epsilon$-sensitive w.r.t $\chi^2$ divergence for $\epsilon = \frac{1}{\delta}$:
\begin{align} \label{24}
    \chi^2(\mathcal{D}(f_{\theta'}), \mathcal{D}(f_{\theta})) & = \int \frac{(p_{f_{\theta'}}(x) - p_{f_{\theta}}(x))^2}{p_{f_{\theta}}(x)} dx \notag \\ & = \int \frac{p(x)^2 ( f_{\theta}(x) - f_{\theta'}(x) - (C_{\theta} - C_{\theta'}))^2}{p(x) (1-f_{\theta}(x) -\delta + C_{\theta})} dx \notag \\
    & \overset{C_{\theta}\ge \delta}{\le} \frac{1}{\delta} \int p(x) \bigg[(f_{\theta}(x) - f_{\theta'}(x))^2 + (C_{\theta} - C_{\theta'})^2 - 2(f_{\theta}(x) - f_{\theta'}(x)) (C_{\theta} - C_{\theta'})\bigg] dx
   \notag  \\ & = \frac{1}{\delta} \bigg[\left(\int p(x) (f_{\theta}(x) - f_{\theta'}(x))^2 dx\right) + (C_{\theta} - C_{\theta'})^2 - 2(C_{\theta} - C_{\theta'})\int p(x)(f_{\theta}(x) - f_{\theta'}(x))dx\bigg]
   \notag  \\ & \overset{(*')}{=} \frac{1}{\delta} \bigg[\left(\int p(x) (f_{\theta}(x) - f_{\theta'}(x))^2 dx\right) + (C_{\theta} - C_{\theta'})^2 - 2(C_{\theta} - C_{\theta'})^2\bigg]
   \notag  \\ & = \frac{1}{\delta} \bigg[\int p(x) (f_{\theta}(x) - f_{\theta'}(x))^2 dx  - (C_{\theta} - C_{\theta'})^2\bigg]
    \notag \\ & \le \frac{1}{\delta} \int p(x) (f_{\theta}(x) - f_{\theta'}(x))^2 dx
    \notag \\ & = \frac{1}{\delta} \|f_{\theta} - f_{\theta'}\|^2
\end{align}

where $(*')$ comes from the fact that $\int p(x)(f_{\theta}(x) - f_{\theta'}(x))dx = C_{\theta} - C_{\theta'}$.

Since $\delta \le C_{\theta} \le 1\ \forall \theta$, it is easy to see that for any $f_{\theta^*}$ and for any $x$:
\begin{equation}
  \frac{1}{2-\delta} \le \frac{p(x)}{p_{f_{\theta^*}}(x)} = \frac{1}{1-g(f_{\theta^*}(x))+C_{\theta^*}} \le \frac{1}{\delta}.
\end{equation}
Consequently, 
\[\frac{1}{2-\delta}\  \mathbb{E}_{p_{f_{\theta^*}}} [(f_{\theta}-f_{\theta'})^2]\le \mathbb{E}_{p} [(f_{\theta}-f_{\theta'})^2] \le \frac{1}{\delta}\  \mathbb{E}_{p_{f_{\theta^*}}} [(f_{\theta}-f_{\theta'})^2]. \]
So the distribution map satisfies the bounded ratio condition for $C = \frac{1}{\delta}$ and $c=\frac{1}{2-\delta}$.

\paragraph{The case where we only resample strategic features.} Suppose that features $x$ are divided into strategic features $x_s$ and non-strategic features $x_f$, i.e. $x = (x_s, x_f)$, and the strategic and non-strategic features are independent. Here, we only resample strategic features with probability $g(f_{\theta}(x))$ which is the probability of rejection. The pdf of $p_{f_{\theta}}$ would be as follows: 
\begin{equation} \label{eq:22}
  p_{f_{\theta}}(x) = p(x) (1-g(f_{\theta}(x))) + \int_{x'_s} p(x'_s, x_f)\  g(f_{\theta}(x'_s, x_f)) \ p_{X_s}(x_s) dx'_s
\end{equation}
Since we only resample strategic features, non-strategic features stay the same as $x_f$.

We can re-write~(\ref{eq:22}) as follows:
\begin{align}
      p_{f_{\theta}}(x) & = p(x) (1-g(f_{\theta}(x))) + \int_{x'_s} p(x'_s, x_f)\  g(f_{\theta}(x'_s, x_f)) \ p_{X_s}(x_s) dx'_s \notag \\
      & = p(x) (1-g(f_{\theta}(x))) + \int_{x'_s} g(f_{\theta}(x'_s, x_f))  p_{X_s}(x'_s) p_{X_f}(x_f) p_{X_s}(x_s)  dx'_s \notag \\
      & = p(x) (1-g(f_{\theta}(x))) + \int_{x'_s} g(f_{\theta}(x'_s, x_f))  p_{X_s}(x'_s) p(x) dx'_s \notag \\
      & = p(x) \Big(( 1 - g(f_{\theta}(x))) + \int_{x'_s} g(f_{\theta}(x'_s, x_f))  p_{X_s}(x'_s)  dx'_s \Big)
\end{align}
where $p_{X_s}$ and $p_{X_f}$ refer to the marginal distributions of strategic and non-strategic features respectively.

Taking $C_{\theta}(x_f) = \int_{x'_s} g(f_{\theta}(x'_s, x_f))  p_{X_s}(x'_s) dx'_s$, $p_{f_{\theta}}(x) = p(x) \Big( 1- g(f_{\theta}(x)) + C_{\theta}(x_f)\Big)$, and we can follow similar steps as \ref{24} to show the $\epsilon$-sensitivity w.r.t $\chi^2$.
\begin{align} \label{eq:strategicRIR}
    \chi^2(\mathcal{D}(f_{\theta'}), \mathcal{D}(f_{\theta})) & = \int \frac{(p_{f_{\theta'}}(x) - p_{f_{\theta}}(x))^2}{p_{f_{\theta}}(x)} dx \notag \\ & = \int \frac{p(x)^2 ( f_{\theta}(x) - f_{\theta'}(x) - (C_{\theta}(x_f) - C_{\theta'}(x_f)))^2}{p(x) (1-f_{\theta}(x) -\delta + C_{\theta}(x_f)} dx \notag \\
    & \overset{C_{\theta}(x_f)\ge \delta}{\le} \frac{1}{\delta} \int p(x) \bigg[(f_{\theta}(x) - f_{\theta'}(x))^2 + (C_{\theta}(x_f) - C_{\theta'}(x_f))^2 - 2(f_{\theta}(x) - f_{\theta'}(x)) (C_{\theta}(x_f) - C_{\theta'}(x_f))\bigg] dx\notag  \\ 
   & = \frac{1}{\delta} \bigg[\left(\int p(x) (f_{\theta}(x) - f_{\theta'}(x))^2 dx\right) \notag \\
   & \quad\quad + \int_{x_f} p_{X_f}(x_f)(C_{\theta}(x_f) - C_{\theta'}(x_f))^2 dx_f \notag \\
   & \quad\quad - 2\int_{x_f}p_{X_f}(x_f)(C_{\theta}(x_f) - C_{\theta'}(x_f))\int_{x_s} p_{X_s}(x_s)(f_{\theta}(x_s, x_f) - f_{\theta'}(x_s, x_f))dx_s dx_f\bigg]
   \notag \\ 
   & \overset{(*'')}{=} \frac{1}{\delta} \bigg[\left(\int p(x) (f_{\theta}(x) - f_{\theta'}(x))^2 dx\right) - \int_{x_f} p_{X_f}(x_f)(C_{\theta}(x_f) - C_{\theta'}(x_f))^2 dx_f\bigg] \notag  \\ 
    & \le \frac{1}{\delta} \int p(x) (f_{\theta}(x) - f_{\theta'}(x))^2 dx
    \notag \\ 
    & = \frac{1}{\delta} \|f_{\theta} - f_{\theta'}\|^2
\end{align}
where $(*'')$ comes from the fact that $C_{\theta}(x_f) - C_{\theta'}(x_f)=\int_{x_s} p_{X_s}(x_s)(f_{\theta}(x_s, x_f) - f_{\theta'}(x_s, x_f))dx_s$. This completes the proof that in the case of only resampling strategic features also, the distribution map is $\frac{1}{\delta}$-sensitive w.r.t $\chi^2$ divergence.
\end{proof}
}

\end{document}